\documentclass[runningheads]{llncs}
\usepackage[T1]{fontenc}
\usepackage{graphicx}
\usepackage{amsfonts}       
\usepackage{makecell} 
\usepackage[export]{adjustbox}
\usepackage{subcaption}
\usepackage{placeins}
\usepackage{mathtools}
\usepackage[dvipsnames]{xcolor}
\usepackage[titlenumbered, ruled,vlined, linesnumbered]{algorithm2e} 
\usepackage{placeins}
\DeclareMathOperator*{\argmax}{\arg\max}

\newcommand{\Var}{\operatorname{Var}}

\PassOptionsToPackage{hyphens}{url}\usepackage{hyperref}

\newtheorem{assumption}{Assumption}

\newcounter{ClemensCounter}

\newcounter{SteffenCounter}

\newcounter{FelixCounter}

\newcounter{PhilippCounter}

\begin{document}
\title{Safe Policy Improvement Approaches and their Limitations}
%
%
\author{Philipp Scholl\inst{1}  \and
Felix Dietrich\inst{2} \and
Clemens Otte \inst{3} \and
Steffen Udluft\inst{3}}
\authorrunning{P. Scholl et al.}
%
\institute{Ludwig-Maximilians University, Munich, Germany
\email{scholl@math.lmu.de}\and
Technical University of Munich, Munich, Germany
\email{felix.dietrich@tum.de}\\\and
Siemens Technology, Munich, Germany\\
\email{\{clemens.otte,steffen.udluft\}@siemens.com}}
\maketitle              
\begin{abstract}

Safe Policy Improvement (SPI) is an important technique for offline reinforcement learning in safety critical applications as it improves the behavior policy with a high probability. We classify various SPI approaches from the literature into two groups, based on how they utilize the uncertainty of state-action pairs. Focusing on the Soft-SPIBB (Safe Policy Improvement with Soft Baseline Bootstrapping) algorithms, we show that their claim of being provably safe does not hold.  Based on this finding, we develop adaptations, the Adv-Soft-SPIBB algorithms, and show that they are provably safe. A heuristic adaptation, Lower-Approx-Soft-SPIBB, yields the best performance among all SPIBB algorithms in extensive experiments on two benchmarks. We also check the safety guarantees of the provably safe algorithms and show that huge amounts of data are necessary such that the safety bounds become useful in practice.

\end{abstract}
\section{Introduction}  \label{sec:introduction}

Reinforcement learning (RL) in industrial control applications such as gas turbine control \cite{schaefer_neural_2007} often requires learning solely from pre-recorded observational data (called offline or batch RL \cite{lange_batch_2012,pmlr-v97-fujimoto19a,DBLP:journals/corr/abs-2005-01643}) to avoid a potentially unsafe online exploration. This is especially necessary when simulations of the system are not available. Since it is difficult to assess the actual quality of the learned policy in this situation \cite{5967358,wang2021what}, Safe Policy Improvement (SPI) \cite{thomas_safe_nodate,nadjahi_safe_2019} is an attractive option. SPI aims to ensure that the learned policy is, with high probability, at least approximately as good as the behavioral policy given, for example, by a conventional controller.

In this paper, we review and benchmark current SPI algorithms and divide them into two classes. Among them, we focus on the class of Soft-SPIBB algorithms \cite{nadjahi_safe_2019} and show that they are, contrary to the authors' claim, not provably safe. Therefore, we introduce the adaptation Adv-Approx-Soft-SPIBB. We also develop the heuristic Lower-Approx-Soft-SPIBB, following an idea presented in Laroche et al.~\cite{laroche_safe_2019}. Additionally, we conduct experiments to test these refined algorithms against their predecessors and other promising SPI algorithms. Here, our taxonomy of the SPI algorithms is helpful, as the two classes of algorithms show different behavior. 
In summary, we extend observations from Scholl et al. \cite{scholl_icaart22} in several ways:
\begin{itemize}
    \item We include a detailed presentation of the competitor SPI algorithms to equip the reader with an in-depth understanding of various SPI mechanisms and our taxonomy.
    \item We present the complete proof of the safety of Adv-Soft-SPIBB with further adaptations from the one from Nadjahi et al. \cite{nadjahi_safe_2019} and discuss the usage of a tighter error bound (Maurer and Pontil \cite{mpeb}) to strengthen the safety theorem.
    \item We test empirically the error bounds of all provably safe algorithms to check their applicability and identify some limitations, as some algorithms need huge amounts of data to produce good policies while maintaining a meaningful error bound at the same time. 
\end{itemize}

The code for the algorithms and experiments can be found in the accompanying repository.\footnote{\url{https://github.com/Philipp238/Safe-Policy-Improvement-Approaches-on-Discrete-Markov-Decision-Processes}}
The next section introduces the mathematical framework necessary for the later sections. Section \ref{sec:related-work} introduces other SPI algorithms, divides them into two classes and explains the most promising ones in more detail. In Section \ref{sec:soft-spibb}, we present the work done by Nadjahi et al. \cite{nadjahi_safe_2019} and show that their algorithms are not provably safe. In Section \ref{sec:advantageous-soft-spibb}, we refine the algorithms of Nadjahi et al. \cite{nadjahi_safe_2019} and prove the safety of one of them. Tests against various competitors on two benchmarks are described in Section \ref{sec:experiments}. The test for the usefulness of the safety guarantees is discussed in Section \ref{sec:limitations-of-the-theory-in-practice}.

\section{Mathematical Framework}

The control problem we want to tackle with reinforcement learning consists of an agent and an environment, modeled as a finite Markov Decision Process (MDP). A finite MDP $M^*$ is represented by the tuple $M^*=(\mathcal{S}, \mathcal{A}, P^*, R^*, \gamma)$, where $\mathcal{S}$ is the finite state space, $\mathcal{A}$ the finite action space, $P^*$ the unknown transition probabilities, $R^*$ the unknown stochastic reward function, the absolute value of which is assumed to be bounded by $R_{max}$, and $0\leq\gamma<1$ is the discount factor.

The agent chooses action $a\in\mathcal{A}$ with probability $\pi(a|s)$ in state $s\in\mathcal{S}$, where $\pi$ is the policy controlling the agent. The return at time $t$ is defined as the discounted sum of rewards $G_t=\sum_{i=t}^{T}\gamma ^{i-t} R^*(s_i,a_i)$, with $T$ the time of termination of the MDP. As the reward function is bounded the return is bounded as well, since $|G_t|\leq\frac{R_{max}}{1-\gamma}$. So, let $G_{max}$ be a bound on the absolute value of the return. The goal is to find a policy $\pi$ which optimizes the expected return, i.e., the state-value function $V_{M^*}^{\pi}(s)=E_\pi[G_t|S_t=s]$ for the initial state $s\in\mathcal{S}$. Similarly, the action-value function is defined as $Q_{M^*}^{\pi}(s,a)=E_\pi[G_t|S_t=s,A_t=a]$.

Given data $\mathcal{D}=(s_j,a_j,r_j,s_j')_{j=1,\dots,n}$ collected by the baseline policy $\pi_b$, let $N_\mathcal{D}(s,a)$ denote the number of visits of the state-action pair $(s,a)$ in $\mathcal{D}$ and $\hat{M}=(\mathcal{S},\mathcal{A},\hat{P},\hat{R},\gamma)$ the Maximum Likelihood Estimator (MLE) of $M^*$ where
\begin{equation}
	\hat{P}(s'|s,a)=\frac{\sum_{(s_j=s,a_j=a,r_j,s_j'=s')\in\mathcal{D}}1}{N_\mathcal{D}(s,a)} \text{ and }   \hat{R}(s,a)=\frac{\sum_{(s_j=s,a_j=a,r_j,s_j')\in\mathcal{D}}r_j}{N_\mathcal{D}(s,a)}.
\end{equation}

\section{Related Work} \label{sec:related-work}

Safety is an overloaded term in Reinforcement Learning, because it can refer to the inherent uncertainty, safe exploration techniques, or parameter uncertainty
\cite{garcia_comprehensive_nodate}. In this paper we focus on the latter. In the following subsection we will introduce a taxonomy of SPI algorithms defined by us, and then go into detail for the most promising methods in the other subsections.

\subsection{Taxonomy} \label{sec:taxonomy}

Many existing Safe Policy Improvement (SPI) algorithms utilize the uncertainty of state-action pairs in one of the two following ways (see also Figure \ref{fig:taxonomy}):
\begin{itemize}
    \item[1.] The uncertainty is applied to the action-value function to decrease the value of uncertain actions. Therefore, these algorithms usually adapt the Policy Evaluation (PE) step.
    \item[2.] The uncertainty is used to restrict the set of policies that can be learned. Therefore, these algorithms usually adapt the Policy Iteration (PI) step.
\end{itemize}

\begin{figure*}[h!]
  \centering
  \includegraphics[width = 1\linewidth]{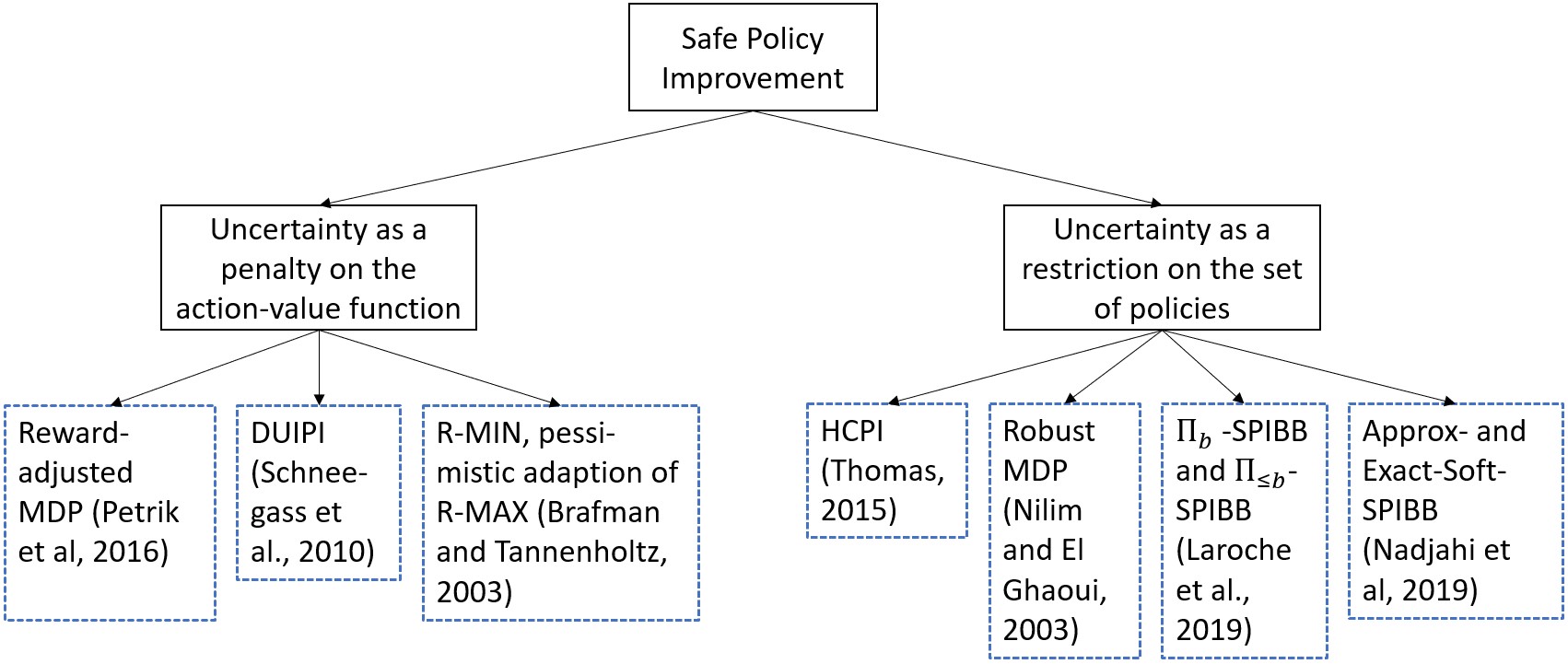}
  \caption{Taxonomy of SPI algorithms~\cite{scholl_icaart22}.}
  \label{fig:taxonomy}
 \end{figure*}

In the remainder of this section we introduce all algorithms shown in Figure \ref{fig:taxonomy}, except for High Confidence Policy Improvement (HCPI) \cite{thomas_safe_nodate} and Robust MDP \cite{robust_mdps}, as they have been shown in Laroche et al. \cite{laroche_safe_2019} and Nadjahi et al. \cite{nadjahi_safe_2019} to be not competitive in practice against many of the other SPI algorithms. We also do not introduce the Soft-SPIBB algorithms, as they are investigated in more detail in Section \ref{sec:soft-spibb}.

\subsection{RaMDP} \label{sec:ramdp}
Petrik et al. \cite{petrik} showed that maximizing the difference between the new policy and a baseline policy on a rectangular uncertainty set of the transition probabilities is NP-hard and, thus, derived the approximation Reward-adjusted MDP (RaMDP). RaMDP applies the uncertainty to penalize the reward and, therefore, the action-value function. Laroche et al. \cite{laroche_safe_2019} extended this algorithm by a hyper-parameter which controls the influence of the uncertainty and modifies the reward matrix to incorporate the uncertainty in the following way:
\begin{equation} \label{eq:mbie-eb}
	\tilde{R}(s,a) = \hat{R}(s,a) -\frac{\kappa}{\sqrt{N_\mathcal{D}(s,a)}},
\end{equation}
where $\kappa>0$ is some hyper-parameter.

\subsection{DUIPI} \label{sec:duipi}

While RaMDP computes the uncertainty simply as a function of the number of visits to a state-action pair, Schneegass et al. \cite{jabin_uncertainty_2010} try a more sophisticated approach to estimate the uncertainty for Diagonal Approximation of Uncertainty Incorporating Policy Iteration (DUIPI) by approximating the variance of the action-value function.

Contrary to most other papers which reduce the reward function to
$R:\mathcal{S}\times\mathcal{A}\rightarrow\mathbb{R}$, they consider
\begin{equation}
	\hat{R}_3:\mathcal{S}\times\mathcal{A}\times\mathcal{S}\rightarrow\mathbb{R},
	\quad
	\hat{R}_3(s,a,s')=\frac{\sum_{(s_j=s,a_j=a,r_j,s_j'=s')\in\mathcal{D}}r_j}{N_\mathcal{D}(s,a,s')}.
\end{equation}
The calculation of the action-value function for some policy $\pi$ in the MDP
with the estimates $\hat{P}$ and $\hat{R}_3$ is classically done by Iterative Policy Evaluation, i.e., by inserting the estimate of the action-value function iteratively into the action Bellman equation \cite{sutton_reinforcement_nodate}:
	\begin{equation} \label{eq:iterative-policy-evaluation}
		Q_{k+1}(s,a)=\sum_{s'}\hat{P}(s'|s,a)(\hat{R}_3(s,a,s')+\gamma \sum_{a'}\pi(a'|s')Q_k(s',a')).
	\end{equation}
The variance of $Q$ is then also
calculated iteratively by using uncertainty propagation, i.e., the fact that 
\begin{equation} \label{eq:uncertainty-propagation}
	\Var(f(X))\approx D\Var(X)D^T
\end{equation}
holds for
a function $f:\mathbb{R}^m\rightarrow\mathbb{R}^n$ with Jacobian
$D_{i,j}=\frac{\partial f_i}{\partial x_j}$.
Schneegass et al. \cite{jabin_uncertainty_2010} present two different approaches in their paper, the first makes use of the
complete covariance matrix, while the second neglects the correlations
between state-action pairs. The authors show that the second
approach is significantly faster and that the correlation is less important for
MDPs with bigger state and action spaces. In the benchmark where it made a
difference, there were 6 states with 2 actions and in the one where it
did not matter anymore there were 18 states and 4 actions. As all
benchmarks considered in this paper use MDPs which are greater than those two,
we concentrate on the second approach called DUIPI (Diagonal
Approximation of Uncertainty Incorporating Policy Iteration). By neglecting the
covariances, Equation \ref{eq:uncertainty-propagation} becomes
\begin{equation} \label{eq:uncertainty-propagation-diagonal}
	\Var(f(X)) \approx \sum_i \left(\frac{\partial f}{\partial X_i}\right)^2\Var(X_i),
\end{equation}
which should be read component-wise. Applying Equation
\ref{eq:uncertainty-propagation-diagonal} to the action-value update in
Equation \ref{eq:iterative-policy-evaluation} yields
\begin{equation} \label{eq:DUIPI-update-cov}
	\begin{gathered}
		\Var(Q^m(s,a)) \approx \sum_{s'}
		(D_{Q,Q}^m)^2\left(\sum_{a'}\pi(a'|s')^2\Var(Q^{m-1}(s',a'))\right)\\
		+ \sum_{s'} D_{Q,\hat{P}}^m \Var(\hat{P}(s'|s,a))
		+ \sum_{s'} D_{Q,\hat{R}_3}^m \Var(\hat{R}_3(s,a,s')))
	\end{gathered}
\end{equation}
with the derivatives
\begin{equation}
	\begin{gathered}
		D_{Q,Q}^m=\gamma\hat{P}(s'|s,a),\:
		D_{Q,\hat{P}}^m=\hat{R}_3(s,a,s')+\gamma \sum_{a'}\pi(a'|s')^2Q^{m-1}(s',a') \\
		\text{and }
		D_{Q,\hat{R}_3}^m = \hat{P}(s'|s,a).
	\end{gathered}
\end{equation}
Equation \ref{eq:DUIPI-update-cov} shows that the variance of the action-value
function at the current step depends on the variance of the action-value
function at the previous step and on the variance of the estimators $\hat{P}$
and $\hat{R}_3$.

To make use of the variance of the action-value function, which is simultaneously computed in the PE step, the authors propose to
define
\begin{equation} \label{eq:DUIPI-uncertainty-q-function}
	Q_u(s,a) = Q^\pi_{\hat{M}}(s,a) - \xi \sqrt{\Var(Q^\pi_{\hat{M}}(s,a))}
\end{equation}
and use the uncertainty incorporating action-value function $Q_u$ instead of
$Q^\pi_{\hat{M}}$ in the PI step. The policy is then chosen in a non greedy way, for details see Schneegass et al. \cite{jabin_uncertainty_2010}. By choosing $\xi>0$, the uncertainty of
one state-action pair has a negative influence on $Q_u$ and the new policy prefers state-action pairs with a low variance.

Assuming that the variance is well approximated, even though uncertainty propagation is used and covariances between the state-action pairs are neglected, and a normal distribution as the prior for the action-value function is close to reality,
yields a confidence interval on the real action-value function $Q$, as
\begin{equation} \label{eq:duipi-safety-bound}
	\mathbb{P}(Q(s,a) > Q_u(s,a)) = 1 - F(\xi)
\end{equation}
holds, with $Q_u$ as defined in Equation \ref{eq:DUIPI-uncertainty-q-function}
and $F$ as the CDF of a standard normal distribution \cite{jabin_uncertainty_2010}.

One question that remains is how to determine the covariances of the estimators
$\hat{P}$ and $\hat{R}_3$. The authors present a Bayesian approach applying the Dirichlet distribution as a prior for this,
which will be used in this paper.

In the experiments in Section \ref{sec:experiments} we observed that DUIPI's
estimation of the transition probabilities results in the possibility that
unobserved state-action pairs are chosen, which is highly dangerous for safe reinforcement learning. For this reason we implemented a
mechanism which checks if a state-action pair has been visited at least once in the training data and if it has not been visited, the probability of the new policy to choose this state-action pair will be set to 0 if possible.
In our experiments this adjustment turned out to yield a strong performance improvement and, thus, we use only this new implementation of DUIPI for the experiments in the Sections \ref{sec:experiments} and \ref{sec:limitations-of-the-theory-in-practice}.

\subsection{R-MIN} \label{sec:r-min}
	
Brafman and Tannenholtz \cite{Brafman2002RMAXA} introduce the R-MAX algorithm as a simple
online algorithm for exploration. Its basic idea is to set the value of a state-action pair to the highest possible value, if it has not been
visited more than $N_\wedge$ times yet, to encourage the algorithm to try these state-action pairs for an efficient
exploration.

The problem considered in this paper is quite opposite to that. First
of all, we are confronted with a batch offline learning problem. Furthermore, we are not interested in exploration but in safety. Consequently, we invert their
technique and instead of setting the value of uncertain state-action pairs to
the best possible value, we set it to the worst possible value, which is why we
renamed the algorithm to R-MIN.
The action-value function, thus, computes as
\begin{equation}\label{eq:r-min-PE}
	Q_{k+1}(s,a)=\begin{dcases}
		\sum_{s',r}p(s',r|s,a)(r+\gamma \sum_{a'}\pi(a'|s')Q_k(s',a')),
	    & \text{if } N_\mathcal{D}(s,a)>N_\wedge\\
    \frac{R_{min}}{1-\gamma}, & \text{if } N_\mathcal{D}(s,a)\leq
	N_\wedge,
	\end{dcases}
\end{equation}
as $\frac{R_{min}}{1-\gamma}$ is the sharpest lower bound on the value function which works for any
MDP, without any additional knowledge about its structure. Note that it is
important to apply Equation \ref{eq:r-min-PE} in every iteration of the PE step
and not only at the end, as the uncertainty of one state-action pair should also
influence the value of state-action pairs leading to the uncertain one with a
high probability.

\subsection{SPIBB} \label{sec:spibb}

Regarding algorithms restricting the policy set, Laroche et al. \cite{laroche_safe_2019} only allow deviations from the baseline policy at a state-action pair if the uncertainty is low enough, otherwise it remains the same. To decide whether the uncertainty of a state-action pair is high, they only regard the times a state-action pair has been visited and use the bootstrapped set $\mathcal{B}$ as the set of uncertain state-action pairs, which has been visited less than some hyperparameter $N_\wedge$:
\begin{equation} \label{eq:def-bootstrapped-set}
	\mathcal{B} =
	\{(s,a)\in\mathcal{S}\times\mathcal{A}:N_\mathcal{D}(s,a)\leq N_\wedge\}.
\end{equation}
Laroche et al. \cite{laroche_safe_2019} propose two algorithms: $\Pi_b$-SPIBB, which is provably safe, and $\Pi_{\leq b}$-SPIBB, which is a heuristic relaxation. $\Pi_b$-SPIBB approximately solves the following optimization problem in its Policy Improvement step:
\begin{subequations}
\begin{equation}
    \pi'=\argmax_\pi\sum_{a\in\mathcal{A}}Q^\pi_{\hat{M}}(s,a)\pi(a|s) \text{, subject to:}
\end{equation}
\begin{equation}
    \pi'(\cdot|s) \text{ being a probability density over }\mathcal{A} \text{ and }
\end{equation}
\begin{equation}
    \pi'(a|s)=\pi_b(a|s), \;\forall (s,a)\in\mathcal{B}.
\end{equation}
\end{subequations}

The fact that a policy computed by $\Pi_b$-SPIBB equals the behavior policy $\pi_b$ for
any state-action pair in the bootstrapped set, enables the authors to prove that
$\Pi_b$-SPIBB produces policies $\pi$ for which
\begin{equation}
	V^\pi_{M^*}(s)-V^{\pi_b}_{M^*}(s)\geq-\frac{4V_{max}}{1-\gamma}
	\sqrt{\frac{2}{N_\wedge}\log\frac{2|\mathcal{S}||\mathcal{A}|2^{|\mathcal{S}|}}{\delta}}
\end{equation}
holds with probability $1-\delta$ for any state $s\in\mathcal{S}$.

Its heuristic variation $\Pi_{\leq b}$-SPIBB does not exhibit any theoretical
safety. However, it performed even better than $\Pi_{b}$-SPIBB in their
benchmarks. Its most concise description is probably via its
optimization problem,
\begin{subequations}
\begin{equation}
    \pi'=\argmax_\pi\sum_{a\in\mathcal{A}}Q^\pi_{\hat{M}}(s,a)\pi(a|s) \text{, subject to:}
\end{equation}
\begin{equation}
    \pi'(\cdot|s) \text{ being a probability density over }\mathcal{A} \text{ and }
\end{equation}
\begin{equation}
    \pi'(a|s)\leq\pi_b(a|s), \;\forall (s,a)\in\mathcal{B}.
\end{equation}
\end{subequations}
This problem only differs in its second constraint to the optimization problem corresponding to $\Pi_b$-SPIBB. While $\Pi_b$-SPIBB
prohibits the change of the behavior policy for uncertain state-action pairs,
$\Pi_{\leq b}$-SPIBB just does not allow to add more weight to them but allows
to reduce it.

\section{Safe Policy Improvement with Soft Baseline Bootstrapping} \label{sec:soft-spibb}
Nadjahi et al. \cite{nadjahi_safe_2019} continue the line of work on the SPIBB algorithms and relax the hard bootstrapping to a softer version, where the baseline policy can be changed at any state-action pair, but the amount of possible change is limited by the uncertainty at this state-action pair. They claim that these new algorithms, called Safe Policy Improvement with Soft Baseline Bootstrapping (Soft-SPIBB), are also provably safe, a claim that is repeated in Simao et al. \cite{Simo2020SafePI} and Leurent \cite{leurent:tel-03035705}. Furthermore, they extend the experiments from Laroche et al. \cite{laroche_safe_2019} to include the Soft-SPIBB algorithms, where the empirical advantage of these algorithms becomes clear. However, after introducing the algorithms and the theory backing up their safety, we will show in Section \ref{sec:soft-spibb-shortcomings-of-the-theory} that they are in fact not provably safe, which motivates our adaptations in Section \ref{sec:advantageous-soft-spibb}.

\subsection{Preliminaries} \label{sec:soft-spibb-preliminaries}

To bound the performance of the new policy, it is necessary to bound the estimate of the transition probabilities $\hat{P}$. By applying Hoelder's inequality, Nadjahi et al. \cite{nadjahi_safe_2019} show that
\begin{equation} \label{eq:bound-P}
	||P(\cdot|s,a)-\hat{P}(\cdot|s,a)||_1\leq e_P(s,a),
\end{equation}
holds with probability $1-\delta$, where 
\begin{equation} \label{eq:error-function-P}		
	e_P(s,a)=\sqrt{\frac{2}{N_\mathcal{D}(s,a)}\log\frac{2|\mathcal{S}||\mathcal{A}|2^{|\mathcal{A}|}}{\delta}}.
\end{equation}
The error function is used to quantify the uncertainty of each state-action pair.
\begin{definition} \label{def:constrained}
	A policy $\pi$ is $(\pi_b,\epsilon,e)$-\emph{constrained} w.r.t.\ a baseline policy
	$\pi_b$, an error function $e$ and a hyper-parameter $\epsilon>0$, if 
	\begin{equation} \label{eq:constrained}
		\sum_{a\in\mathcal{A}}e(s,a)|\pi(a|s)-\pi_b(a|s)|\leq\epsilon
	\end{equation}
	holds for all states $s\in\mathcal{S}$.
\end{definition}
Therefore, if a policy $\pi$ is $(\pi_b,\epsilon,e)$-constrained, it means that the $l^1$-distance between $\pi$ and $\pi'$, weighted by some error function $e$, is at most $\epsilon$. 

\subsection{Algorithms} \label{sec:soft-spibb-algorithms}

The new class of algorithms Nadjahi et al. \cite{nadjahi_safe_2019} introduce aims at solving the following constrained optimization problem during the Policy Improvement step:
\begin{subequations} \label{eq:optimization-problem-PI}
\begin{equation}
    \pi'=\argmax_\pi\sum_{a\in\mathcal{A}}Q^\pi_{\hat{M}}(s,a)\pi(a|s) \text{, subject to:}
\end{equation}
\begin{equation}
    \pi'(\cdot|s) \text{ being a probability density over }\mathcal{A} \text{ and }
\end{equation}
\begin{equation}
    \pi^{(i+1)} \text{ being } (\pi_b,\epsilon, e)\text{-constrained}.
\end{equation}
\end{subequations}
This computation leads to the optimal---w.r.t. the action-value function of the previous policy---$(\pi_b,\epsilon, e)$-constrained policy. The two algorithms introduced in Nadjahi et al. \cite{nadjahi_safe_2019}, which numerically solve this optimization problem, are Exact-Soft-SPIBB and Approx-Soft-SPIBB. The former solves the linear formulation of the constrained problem by a linear program \cite{linear-program} and the latter uses a budget calculation for the second constraint to compute an approximate solution. In experiments, it is shown that both algorithms achieve similar performance, but Exact-Soft-SPIBB takes considerably more time \cite{nadjahi_safe_2019}.

\subsection{The Safety Guarantee} \label{sec:soft-spibb-safety-guarantee}

Nadjahi et al. \cite{nadjahi_safe_2019} derive the theoretical safety of their algorithms from two theorems. The first theorem, however, needs an additional property their algorithms do not fulfill. For that reason we will move its discussion to Section \ref{sec:advantageous-soft-spibb}, where we will also introduce refined algorithms which fulfill said property. 
To make up for this property, Nadjahi et al. \cite{nadjahi_safe_2019} use Assumption \ref{ass:assumption-1} to prove Theorem \ref{th:theorem-2}.
\begin{assumption} \label{ass:assumption-1}
    There exists a constant $\kappa<\frac{1}{\gamma}$ such
	that, for all state-action pairs $(s,a)\in\mathcal{S}\times\mathcal{A}$, the
	following holds:
    \begin{equation} \label{eq:assumption-1}
	    \sum_{s',a'}e_P(s',a')\pi_b(a'|s')P^*(s'|s,a)\leq\kappa e_P(s,a)
    \end{equation}
\end{assumption}
Interpreting $\pi_b(a'|s')P^*(s'|s,a)$ as the probability of observing the state-action pair $(s',a')$ after observing $(s,a)$ we can rewrite Equation \ref{eq:assumption-1} to 
\begin{equation} \label{eq:assumption-1-reformulated}
	E_{P,\pi_b}[e_P(S_{t+1},A_{t+1})|S_t=s,A_t=a]\leq\kappa e_P(s,a),
\end{equation}
which shows that Assumption \ref{ass:assumption-1} assumes an upper bound on the uncertainty of the next state-action pair dependent on the uncertainty of the current one. Intuitively this makes sense, but we show in the next section that the bound does not hold in general. However, using this assumption Nadjahi et al. \cite{nadjahi_safe_2019} prove Theorem \ref{th:theorem-2} which omits the advantageous assumption of the new policy.
\begin{theorem} \label{th:theorem-2}
    Under Assumption 1, any
		$(\pi_b,\epsilon,e_P)$-constrained policy $\pi$ satisfies the following
		inequality in every state $s$ with probability at least $1-\delta$:\\
	\begin{gather}
		V^\pi_{M^*}(s)-V^{\pi_b}_{M^*}(s)\geq V^\pi_{\hat{M}}(s)-V^{\pi_b}_{\hat{M}}(s) +\nonumber\\
		2||d^\pi_M(\cdot|s)-d^{\pi_b}_M(\cdot|s)||_1V_{max}
		-\frac{1+\gamma}{(1-\gamma)^2(1-\kappa\gamma)}\epsilon V_{max}
	\end{gather}
\end{theorem}
Here, $d^\pi_M(s'|s)=\sum_{t=0}^\infty\gamma^t\mathbb{P}(S_t=s'|S_t\sim P\pi S_{t-1},S_0=s)$ denotes the expected discounted sum of visits to $s'$ when starting in $s$.

\subsection{Shortcomings of the Theory} \label{sec:soft-spibb-shortcomings-of-the-theory}

As explained above, the theoretical guarantees, Nadjahi et al. \cite{nadjahi_safe_2019} claim for the Soft-SPIBB algorithms, stem from Theorem \ref{th:theorem-2}. However, we now show that Assumption \ref{ass:assumption-1} does not hold for any $0<\gamma<1$.
\begin{theorem}	\label{th:assumption-1}
	Let the discount factor $0<\gamma<1$ be arbitrary. Then there exists an MDP $M$
	with transition probabilities $P$ such that for any behavior policy $\pi_b$ and
	any data set $\mathcal{D}$, which contains every state-action pair at least once, it holds that, for all $0<\delta<1$,
	\begin{equation} \label{eq:assumption-1-theorem}
		\sum_{s',a'}e_P(s',a')\pi_b(a'|s')P(s'|s,a) > \frac{1}{\gamma} e_P(s,a).
	\end{equation}
	This means that Assumption 1 can, independent of the discount factor, not be
	true for all MDPs.
\end{theorem}
\begin{figure}[!h]
  \centering
  \includegraphics[width = 3.5cm]{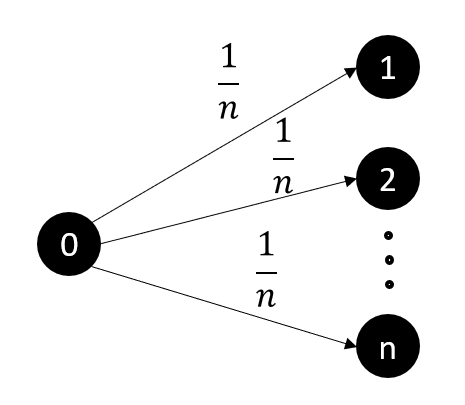}
  \caption{MDP with $n+1$ states, $n$ of them are final states and in the non-final state, there is only $1$ action, leading to one of the others with equal probability.\cite{scholl_icaart22}}
  \label{fig:example-mdp-assumption-1}
\end{figure}
\begin{proof}
 	Let $0<\gamma<1$ be arbitrary and $n\in\mathbb{N}$ be such that
	$\sqrt{n}>\frac{1}{\gamma}$.
		Let $M$ be the MDP displayed in Figure \ref{fig:example-mdp-assumption-1}. It
	has $n+1$ states, from which $n$ states are terminal states, labeled $1$, $2$,
	..., $n$. In the only non-terminal state $0$, there is only one action available
	and choosing it results in any of the terminal states with probability
	$\frac{1}{n}$. As there is only one
	action, one can omit the action in the notation of $e_P$ and there is only one possible behavior policy. So, Equation
	\ref{eq:assumption-1-theorem} can be reduced to
	\begin{equation} \label{eq:assumption-1-proof-one-action}
		\sum_{i=1}^n\frac{e_P(i)}{n} > \frac{e_P(0)}{\gamma} .
	\end{equation}
	Now, we show that 
	\begin{equation} \label{eq:assumption-1-proof-goal}
		\sum_{i=1}^n\frac{e_P(i)}{n} \geq \sqrt{n}e_P(0),
	\end{equation}
	which implies Equation
	\ref{eq:assumption-1-proof-one-action} as
	$\sqrt{n}>\frac{1}{\gamma}$. Let $\mathcal{D}$ denote the data collected on this MDP such that every state has been visited at least once. Thus, $N_\mathcal{D}(i)>0$\textemdash{}the number of visits to state $i$\textemdash{}holds for every $i$. Equation \ref{eq:assumption-1-proof-goal} is equivalent to 
	\begin{equation} \label{eq:assumption-1-proof-equivalent-goal}
		\frac{1}{n}\sum_{i=1}^n\frac{1}{\sqrt{N_\mathcal{D}(i)}} \geq
		\frac{\sqrt{n}}{\sqrt{N}}
	\end{equation}
	where $N=N_\mathcal{D}(0)=\sum_{i=1}^nN_\mathcal{D}(i)$.  Equation \ref{eq:assumption-1-proof-equivalent-goal} follows by applying Jensen's inequality once for the convex function $x\mapsto\frac{1}{x}$, restricted to $x>0$, and once for the concave function $x\mapsto\sqrt{x}$, also restricted to $x>0$:
	\begin{gather} \label{eq:end-proof-ass-1}
	    \frac{1}{n}\sum_{i=1}^n\frac{1}{\sqrt{N_\mathcal{D}(i)}} \geq 
	    \frac{1}{\frac{1}{n}\sum_{i=1}^n\sqrt{N_\mathcal{D}(i)}} \geq \nonumber
	    \frac{1}{\sqrt{\frac{1}{n}\sum_{i=1}^n N_\mathcal{D}(i)}} = 
	    \frac{1}{\sqrt{\frac{N}{n}}} = 
	    \frac{\sqrt{n}}{\sqrt{N}}.
	\end{gather}
	This concludes the proof.
\end{proof}
	The class of MDPs used in the proof and depicted in Figure
	\ref{fig:example-mdp-assumption-1} gives a good impression what kind of
	constellations are critical for Assumption 1. An MDP does not have to exhibit exactly
	the same structure to have similar effects, it might already be enough if there is
	a state-action pair from which a lot of different state-action pairs are
	exclusively accessible. 
	
	Although Assumption 1 is invalid in its
	generality shown at some specific class of MDPs, it might still hold on simple
	MDPs which are not built in order to disprove Assumption 1. One consideration
	here is that $n$ does not need to be especially big as the proof only required
	$\sqrt{n}>\frac{1}{\gamma}$. So, for any $\gamma>\frac{1}{\sqrt{2}}\approx
	0.707$ it suffices to choose $n=2$. 
	
    Furthermore, we tested Assumption \ref{ass:assumption-1} empirically on the Random MDPs benchmark considered in \cite{nadjahi_safe_2019} where we found for no discount factor greater than 0.6 a baseline policy and data set such that the assumption holds for all state-action pairs.\footnote{\url{https://github.com/Philipp238/Safe-Policy-Improvement-Approaches-on-Discrete-Markov-Decision-Processes/blob/master/auxiliary\_tests/assumption\_test.py}}
    Consequently, we conclude that Assumption \ref{ass:assumption-1} is not reasonable and, thus, Theorem \ref{th:theorem-2} cannot be relied upon.

\section{Advantageous Soft-SPIBB} \label{sec:advantageous-soft-spibb}

In the last section we have shown that there is no safety guarantee for the Soft-SPIBB algorithms. However, Nadjahi et al. \cite{nadjahi_safe_2019} proved another theorem which we present in its corrected form together with adapted algorithms to which it is applicable.

\subsection{Preliminaries}
We follow Nadjahi et al. \cite{nadjahi_safe_2019} and define the error function $e_Q$ as
	\begin{equation} \label{eq:error-function-q}
		e_Q(s,a)=\sqrt{\frac{2}{N_\mathcal{D}(s,a)}\log\frac{2|\mathcal{S}||\mathcal{A}|}{\delta}},
	\end{equation}
since that way we can bound the difference between the true action-value function and its estimate. Contrary to them, we use a different estimate as the bound again relies on Hoeffding's inequality \cite{Hoeffding} which can only be applied to a sum of random variables and, to the best of our knowledge, we do not know how it can be applied to $Q^{\pi_b}_{\hat{M}}(s,a)$.

For our proof it is necessary to use the Monte-Carlo estimate for
$Q^{\pi_b}_M(s,a)$ from data $\mathcal{D}$
\begin{equation} \label{eq:mc-estimate-q}
	\hat{Q}^{\pi_b}_\mathcal{D}(s,a)=\frac{1}{n}\sum_{i=1}^{n}G_{t_i}
\end{equation}
where $t_1,...,t_n$ are times such that $(S_{t_i},A_{t_i})=(s,a)$, for all
$i=1,...,n$. As $E[G_{t_i}]=Q^{\pi_b}_M(s,a)$ it is possible to apply the Hoeffding's inequality for $\hat{Q}^{\pi_b}_\mathcal{D}$, similar to how it is done in Nadjahi et al. \cite{nadjahi_safe_2019}.
\begin{lemma}  \label{lem:e_Q}
		For $e_Q$ as defined in Equation \ref{eq:error-function-q}, it holds for an
		arbitrary MDP $M$ with $G_{max}$ as an upper bound on the absolute value of the
		return in any state-action pair of any policy and behavior policy $\pi_b$, that
		\begin{equation} \label{eq:lem-e_Q}
			\mathbb{P}_\mathcal{D}\left(\forall (s,a)\in\mathcal{S}\times\mathcal{A}:
			|Q^{\pi_b}_M(s,a)-\hat{Q}^{\pi_b}_\mathcal{D}(s,a)|\leq e_Q(s,a)
			G_{max}\right)\geq 1-\delta,
		\end{equation}
		if the $G_{t_i}$ used to estimate $\hat{Q}^{\pi_b}_\mathcal{D}(s,a)$ are
		independent.
		Here, the subscript $\mathcal{D}$ of $\mathbb{P}$ emphasizes that the dataset
		generated by $\pi_b$ is the random quantity. 
	\end{lemma}
	\begin{proof}
		First, note that Hoeffding's inequality implies that 
		\begin{equation} \label{eq:hoeffding-absolute-value}
			\mathbb{P}(|\sum_{i=1}^nX_i-E[\sum_{i=1}^nX_i]|\geq t) \leq 2 \exp(-2nt^2),
		\end{equation}
		for $X_i\in[0,1]$ independent.
		Fix the state-action pair $(s,a)\in\mathcal{S}\times\mathcal{A}$. Then,
		\begin{equation}
			\begin{gathered}
				\mathbb{P}_\mathcal{D}\left(
				\left|Q^{\pi_b}_M(s,a)-\hat{Q}^{\pi_b}_\mathcal{D}(s,a)\right|> e_Q(s,a)
				G_{max}\right)\\
				=\mathbb{P}_\mathcal{D}\left( \left|\frac{(Q^{\pi_b}_M(s,a) +
					G_{max})-(\hat{Q}^{\pi_b}_\mathcal{D}(s,a)+G_{max})}{2G_{max}}\right|>\sqrt{\frac{1}{2N_\mathcal{D}(s,a)}\log\frac{2|\mathcal{S}||\mathcal{A}|}{\delta}}\right)\\
				\leq2\exp\left(-2N_\mathcal{D}(s,a)\frac{1}{2N_\mathcal{D}(s,a)}\log\frac{2|\mathcal{S}||\mathcal{A}|}{\delta}\right)
				= \frac{\delta}{|\mathcal{S}||\mathcal{A}|},
			\end{gathered}
		\end{equation}
		where it is possible to apply Hoeffding's inequality in the second step since
		\begin{equation}
			\frac{\hat{Q}^{\pi_b}_\mathcal{D}(s,a)+G_{max}}{2G_{max}}=\frac{1}{n}\sum_{i=1}^{n}\frac{G_{t_i}+G_{max}}{2G_{max}}
		\end{equation}
		is the empirical mean of independent variables which are bounded in $[0,1]$ as
		$G_{t_i}$ are independent and bounded in $[-G_{max},G_{max}]$ and
		\begin{equation}
			E[\hat{Q}^{\pi_b}_\mathcal{D}(s,a)]= E[G_{t_i}]= Q^{\pi_b}_M(s,a).
		\end{equation}
	\end{proof}
Another change in this inequality compared to Nadjahi et al. \cite{nadjahi_safe_2019} is, that we get $G_{max}$ instead of $V_{max}$, since $G_t$ does in general not lie in the interval $[-V_{max},V_{max}]$.

The first theorem from Nadjahi et al. \cite{nadjahi_safe_2019} relies on the algorithms to produce \emph{advantageous} policies. This property has to be slightly generalized to make it applicable to the Monte Carlo estimate of the action value function. 

\begin{definition} \label{def:advantageous-wrt-function}
	A policy $\pi$ is $\pi_b$-\emph{advantageous w.r.t.\ the function}
	$Q:\mathcal{S}\times\mathcal{A}\rightarrow\mathbb{R}$, if 
	\begin{equation} \label{eq:def-advantageous-wrt-function}
		\sum_a Q(s,a)\pi(a|s)\geq\sum_a Q(s,a)\pi_b(a|s)
	\end{equation}
	holds for all states $s\in\mathcal{S}$. 
\end{definition}
Interpreting $Q$ as an action-value function, Definition \ref{def:advantageous-wrt-function} yields that the policy $\pi$ chooses higher-valued actions than policy $\pi'$ for every state.

\subsection{Algorithms}
In this subsection we introduce the adaptation Adv-Approx-Soft-SPIBB which produces ($\pi_b,e_Q,\epsilon)$-constrained and $\pi_b$-advantageous w.r.t. $\hat{Q}^{\pi_b}_{\mathcal{D}}$ policies and is, thus, provably safe by Theorem \ref{th:theorem-1} in Section \ref{sec:soft-spibb-safety-guarantee}. Additionally, we present the heuristic adaptation Lower-Approx-Soft-SPIBB. As both algorithms function similarly to their predecessors by constraining the policy set, they also belong to the category ``Uncertainty as a restriction on the set of policies'' in the taxonomy in Figure \ref{fig:taxonomy}.

\subsubsection{Adv-Approx-Soft-SPIBB}
The advantageous version of the Soft-SPIBB algorithms solves the following problem in the Policy Improvement (PI) step:
\begin{subequations} \label{eq:optimization-problem-PI-adv}
\begin{equation}
    \pi'=\argmax_\pi\sum_{a\in\mathcal{A}}Q^\pi_{\hat{M}}(s,a)\pi(a|s) \text{, subject to:}
\end{equation}
\begin{equation}
    \pi'(\cdot|s) \text{ being a probability density over }\mathcal{A}, 
\end{equation}
\begin{equation}
    \pi^{(i+1)} \text{ being } (\pi_b,\epsilon, e)\text{-constrained} \text{ and }
\end{equation}
\begin{equation} \label{eq:optimization-problem-PI-adv-constraint-3}
    \pi^{(i+1)} \text{ being } \pi_b\text{-advantageous w.r.t. }
\hat{Q}^{\pi_b}_\mathcal{D}.
\end{equation}
\end{subequations}

The original Soft-SPIBB algorithms solve this optimization problem without the constraint in Equation \ref{eq:optimization-problem-PI-adv-constraint-3} as shown in Section \ref{sec:soft-spibb-algorithms}. Adv-Approx-Soft-SPIBB works exactly as its predecessor Approx-Soft-SPIBB except that it keeps an additional budgeting variable ensuring that the new policy is $\pi_b$-advantageous w.r.t. $\hat{Q}^{\pi_b}_\mathcal{D}$.
The derivation of a successor algorithm of Exact-Soft-SPIBB is straightforward since the constraint in Equation \ref{eq:optimization-problem-PI-adv-constraint-3} is linear, however, we observed for Exact-Soft-SPIBB and its successor numerical issues, so, we omit them in the experiments in Section \ref{sec:experiments}.

\subsubsection{Lower-Approx-Soft-SPIBB}

To introduce the heuristic adaptation of Approx-Soft-SPIBB we need a relaxed version of the constrainedness property.
\begin{definition} \label{def:lower-constrained}
	A policy $\pi$ is $(\pi_b,\epsilon,e)$-\emph{lower-constrained} w.r.t.\ a baseline
	policy $\pi_b$, an error function $e$, and a hyper-parameter $\epsilon$, if 
	\begin{equation} \label{eq:lower-constrained}
		\sum_{a\in\mathcal{A}}e(s,a) \max\{0, \pi(a|s)-\pi_b(a|s)\}\leq\epsilon
	\end{equation}
	holds for all states $s\in\mathcal{S}$.
\end{definition}

This definition does not punish a change in uncertain state-action pairs if the probability of choosing it is decreased, which follows the same logic as the empirically very successful adaptation $\Pi_{\leq b}$-SPIBB \cite{laroche_safe_2019}. The optimization problem solved by Lower-Approx-Soft-SPIBB is the following:
\begin{subequations} \label{eq:optimization-problem-PI-lower-approx-soft}
\begin{equation}
    \pi'=\argmax_\pi\sum_{a\in\mathcal{A}}Q^\pi_{\hat{M}}(s,a)\pi(a|s) \text{, subject to:}
\end{equation}
\begin{equation}
    \pi'(\cdot|s) \text{ being a probability density over }\mathcal{A} \text{ and }
\end{equation}
\begin{equation}
    \pi^{(i+1)} \text{ being } (\pi_b,\epsilon, e)\text{-lower-constrained}.
\end{equation}
\end{subequations}

Even though Lower-Approx-Soft-SPIBB is\textemdash{}just as its predecessor Approx-Soft-SPIBB\textemdash{}not provably safe, the experiments in Section \ref{sec:experiments} show that it performs empirically the best out of the whole SPIBB family.
\subsection{Safety Guarantee}
In this subsection, we prove Theorem \ref{th:theorem-1} which states that algorithms that produce $(\pi_b,\epsilon,e_Q)$-constrained policies which are $\pi_b$-advantageous w.r.t.\ $\hat{Q}^{\pi_b}_{\mathcal{D}}$\textemdash{}as Adv-Approx-Soft-SPIBB does\textemdash{}are safe. Before that, it is necessary to introduce some notation from the predecessor paper \cite{laroche_safe_2019}. To enable the use of matrix operations all the quantities introduced so far are denoted as vectors and matrices in the following. The value functions become
\begin{equation} \label{eq:theorem-1-notation-action-value-function}
	Q^\pi=(Q^\pi(1,1),...,Q^\pi(1,|\mathcal{A}|),Q^\pi(2,1),...,Q^\pi(|\mathcal{S}|,|\mathcal{A}|))\in\mathbb{R}^{|\mathcal{S}||\mathcal{A}|}
\end{equation}
and
\begin{equation} \label{eq:theorem-1-notation-state-value-function}
	V^\pi=(V^\pi(1),...,V^\pi(|\mathcal{S}|))\in\mathbb{R}^{|\mathcal{S}|}.
\end{equation}
Similarly, the reward vector is
\begin{equation} \label{eq:theorem-1-notation-reward}
	R=(R(1,1),...,R(1,|\mathcal{A}|),R(2,1),...,R(|\mathcal{S}|,|\mathcal{A}|))\in\mathbb{R}^{|\mathcal{S}||\mathcal{A}|}.
\end{equation}
The policy is denoted by
\begin{equation} \label{eq:theorem-1-notation-policy}
	\pi=(\pi_{\cdot 1} \hdots \pi_{\cdot |\mathcal{S}|})\in\mathbb{R}^{|\mathcal{S}||\mathcal{A}|\times|\mathcal{S}|},
\end{equation}
with columns $\pi_{kj}=\pi(k|j)$ for $(j-1)|\mathcal{A}|+1\leq k \leq j|\mathcal{A}|$ and $\pi_{kj}=0$ otherwise. Lastly, the transition probabilities are
\begin{equation} \label{eq:theorem-1-notation-transition-probabilities}
	P=\left(\begin{array}{ccccc}
		P\left(1\bigr\rvert1,1\right) & \cdots & P\left(1\bigr\rvert1,|\mathcal{A}|\right) & \cdots &
		P\left(1\bigr\rvert|\mathcal{S}|,|\mathcal{A}|\right)\\
		\vdots & & \vdots & & \vdots\\
		P\left(|\mathcal{S}|\bigr\rvert1,1\right) & \cdots & P\left(|\mathcal{S}|\bigr\rvert1,|\mathcal{A}|\right) & \cdots &
		P\left(|\mathcal{S}|\bigr\rvert|\mathcal{S}|,|\mathcal{A}|\right)
		
	\end{array}\right)\in\mathbb{R}^{|\mathcal{S}|\times|\mathcal{S}||\mathcal{A}|}.
\end{equation}
Note that we omit the current MDP $M$ in all these notations, as we always
do, if it is clear or not relevant which MDP is currently used.
Using this notation, we can present Proposition 1 without its proof from Nadjahi et al. \cite{nadjahi_safe_2019}.
\begin{proposition} \label{prop:Proposition-1}
	Let $\pi_1$ and $\pi_2$ be two policies on an MDP $M$. Then
	\begin{equation} \label{eq:prop}
		V^{\pi_1}_M - V^{\pi_2}_M = Q^{\pi_2}_M(\pi_1-\pi_2)d^{\pi_1}_M
	\end{equation}
	holds.
\end{proposition}

Proposition \ref{prop:Proposition-1} yields a decomposition of the
difference of the state value function for two different policies on the same
MDP which is utilized to prove the corrected version of Theorem 1 from
Nadjahi et al. \cite{nadjahi_safe_2019}.

\begin{theorem}\label{th:theorem-1}
	For any $(\pi_b,\epsilon,e_Q)$-constrained policy that is $\pi_b$-advantageous
	w.r.t.\ $\hat{Q}^{\pi_b}_{\mathcal{D}}$, which is estimated with independent
	returns for each state-action pair, the following inequality holds:
	\begin{equation} \label{eq:theorem-1}
		\mathbb{P}_\mathcal{D}\left(\forall
		s\in\mathcal{S}:V_{M^*}^{\pi}(s)-V_{M^*}^{\pi_b}(s)\geq-\frac{\epsilon
			G_{max}}{1-\gamma}\right)\geq 1-\delta,\nonumber
	\end{equation}
	where $M^*$ is the true MDP on which the data $\mathcal{D}$ gets sampled by the
	baseline policy $\pi_b$, $0\leq\gamma<1$ is the discount factor, and $\delta>0$
	is the safety parameter for $e_Q$.
\end{theorem}
\begin{proof}
		We start by applying Proposition \ref{prop:Proposition-1}:
		\begin{equation} \label{eq:theorem-1-proof-2-summands}
			\begin{gathered}
				V^{\pi}_M - V^{\pi_b}_M = Q^{\pi_b}_M(\pi-\pi_b)d^{\pi}_M = (Q^{\pi_b}_M -
				\hat{Q}^{\pi_b}_\mathcal{D} + \hat{Q}^{\pi_b}_\mathcal{D})(\pi-\pi_b)d^{\pi}_M\\
				= (Q^{\pi_b}_M - \hat{Q}^{\pi_b}_\mathcal{D})(\pi-\pi_b)d^{\pi}_M +
				\hat{Q}^{\pi_b}_\mathcal{D}(\pi-\pi_b)d^{\pi}_M.
			\end{gathered}
		\end{equation}
		The goal of this proof is to lower bound these two summands. For the first
		summand, notice that it is a row vector and, so, a bound on the $\ell^1$ norm gives a
		simultaneous bound for all states. Thus, applying a  vector-matrix version of
		Hoelder's inequality, which is proven in Lemma \ref{lem:hoelder-variation} later on, yields
		\begin{equation} \label{eq:hoelder-application}
			\begin{gathered}
				||(Q^{\pi_b}_M - \hat{Q}^{\pi_b}_\mathcal{D})(\pi-\pi_b)d^{\pi}_M||_1 \leq
				||((Q^{\pi_b}_M -
				\hat{Q}^{\pi_b}_\mathcal{D})(\pi-\pi_b))^T||_\infty||d^{\pi}_M||_1,
			\end{gathered}
		\end{equation}
		where all the norms refer to matrix norms.
		Using Lemma \ref{lem:e_Q} and that $\pi$ is  $(\pi_b,e_Q,\epsilon)$-constrained yields that
		\begin{equation} \label{eq:theorem-1-proof-summand-1-factor-1}
			\begin{gathered}
				||\left((Q^{\pi_b}_M-\hat{Q}^{\pi_b}_\mathcal{D})(\pi-\pi_b)\right)^T||_\infty\leq
				\max_s\sum_a|Q^{\pi_b}_M(s,a) -
				\hat{Q}^{\pi_b}_\mathcal{D}(s,a)||\pi(a|s)-\pi_b(a|s)|\\
				\leq \max_s\sum_ae_Q(s,a)G_{max}|\pi(a|s)-\pi_b(a|s)|
				\leq \epsilon G_{max}
			\end{gathered}
		\end{equation}
		holds with probability $1-\delta$. The next step is to compute
		\begin{equation} \label{eq:theorem-1-proof-summand-1-factor-2}
			\begin{gathered}
				||d^{\pi}_M||_1 = \max_s\sum_{s'}\sum_{t=0}^{\infty}\gamma^t
				\mathbb{P}(S_t=s'|X_t\sim P\pi S_t, S_0=s)\\
				= \max_s\sum_{t=0}^{\infty}\gamma^t \sum_{s'}\mathbb{P}(S_t=s'|X_t\sim P\pi
				S_t, S_0=s) = \max_s\sum_{t=0}^{\infty}\gamma^t = \frac{1}{1-\gamma}.
			\end{gathered}
		\end{equation}
		Thus, the first summand can be lower bounded by $-\frac{\epsilon
			G_{max}}{1-\gamma}$ with probability $1-\delta$. To lower bound the second
		summand, note that all entries of $d^\pi_M$ are non-negative, so Theorem
		\ref{th:theorem-1} follows if $\hat{Q}^{\pi_b}_\mathcal{D}(\pi-\pi_b)(s)\geq0$
		for every state $s\in\mathcal{S}$. So, let $s\in\mathcal{S}$ be arbitrary. Then,
		\begin{equation} 
			\begin{gathered}
				\hat{Q}^{\pi_b}_\mathcal{D}(\pi-\pi_b)(s) = \sum_a
				\hat{Q}^{\pi_b}_\mathcal{D}(s,a)(\pi(a|s)-\pi_b(a|s))
				\geq 0
			\end{gathered}
		\end{equation}
		holds, since $\pi$ is  $\pi_b$-advantageous w.r.t.
		$\hat{Q}^{\pi_b}_\mathcal{D}$.
	\end{proof}
	
	As mentioned in the proof, it remains to show that the inequality in Equation
	\ref{eq:hoelder-application} holds.
		
	\begin{lemma} \label{lem:hoelder-variation}
		Let $n\in\mathbb{N}$, $v\in\mathbb{R}^n$ and $A\in\mathbb{R}^{n\times n}$ be
		arbitrary. Then,
		\begin{equation}
			||v^TA||_1\leq ||v||_\infty||A||_1
		\end{equation}
		holds, where all norms refer to matrix norms.
	\end{lemma}
	\begin{proof}
		Denote the columns of $A$ by $a^1,...,a^n$. Then,
		\begin{equation} 
			\begin{gathered}
				||v^TA||_1 = ||(v^Ta^1,...,v^Ta^n)||_1 = \max_{i=1,...,n}\{|v^Ta^i|\}\\
				\leq \max_{i=1,...,n}\{||v||_\infty ||a^i||_1\} =||v||_\infty
				\max_{i=1,...,n}\{ ||a^i||_1\} = ||v||_\infty||A||_1
			\end{gathered}
		\end{equation}
		where we used the definition of the $\ell^1$ norm for matrices as the maximum
		absolute column sum in the second and the last step and Hoeffding's inequality in the third step.
	\end{proof}

One possibility to improve Theorem \ref{th:theorem-1} is to increase the
sharpness of the error function. For this, note that for Theorem \ref{th:theorem-1} it is sufficient that the error function fulfills the inequality from Lemma \ref{lem:e_Q}. One disadvantage of using
Hoeffding's inequality to define the error function $e_Q$ as in Equation \ref{eq:error-function-q} is that it only incorporates the number of observations
to estimate the uncertainty, while it
neglects additional knowledge of the data, e.g., its variance. It should be clear that less observations of a quantity are necessary to
estimate it with high confidence, if the quantity has a low variance. This is
incorporated in the bound derived by Maurer and Pontil \cite{mpeb}.
\begin{lemma} \label{lem:mpeb}
	\textbf{(Maurer and Pontil)} Let $X_1,...,X_n$ be i.i.d. variables with values
	in $[0,1]$ and denote with $X$ the random vector $X=(X_1,...,X_n)$. Then,
	\begin{equation} \label{eq:mpeb}
		\mathbb{P}\left(\bar{X}-E[\bar{X}]\leq
		\sqrt{\frac{2\widehat{\Var}(X)\log\frac{2}{\delta}}{n}} +
		\frac{7\log\frac{2}{\delta}}{3(n-1)} \right) \geq 1 - \delta
	\end{equation}
	for any $\delta > 0$. Here, $\widehat{\Var}(X) = \frac{1}{n(n-1)} \sum_{1\leq i < j \leq n} (Z_i - Z_j)^2$ denotes the sample variance.
\end{lemma}

Following Lemma \ref{lem:mpeb} we define
\begin{equation} \label{eq:error-function-q-mpeb}
	e_Q^{B}(s,a) = 2 \left(\sqrt{\frac{2 \widehat{\Var}(\frac{G 	}{2G_{max}})\log(\frac{4|\mathcal{S}||\mathcal{A}|}{\delta})}{N_\mathcal{D}(s,a)}}
	+
	\frac{7\log(\frac{4|\mathcal{S}||\mathcal{A}|}{\delta})}{3(N_\mathcal{D}(s,a)-1)}
	\right)
\end{equation}
with $G=(G_{t_1},...,G_{t_n})$. With exactly the same proof as for Lemma
\ref{lem:e_Q} for $e_Q$, it is possible to show the same inequality for $e_Q^B$, if one
applies the empirical version of Bennett's inequality instead of Hoeffding's
inequality:

\begin{lemma}  \label{lem:e_Q^B}
	For $e_Q^B$ as defined in Equation \ref{eq:error-function-q-mpeb}, it holds for
	an arbitrary MDP $M$ with $G_{max}$ as an upper bound on the absolute value of
	the return in any state-action pair of any policy and behavior policy $\pi_b$,
	that
	\begin{equation} \label{eq:lem-e_Q^B}
		\mathbb{P}_\mathcal{D}\left(\forall (s,a)\in\mathcal{S}\times\mathcal{A}:
		|Q^{\pi_b}_M(s,a)-\hat{Q}^{\pi_b}_\mathcal{D}(s,a)|\leq e_Q^B(s,a)
		G_{max}\right)\geq 1-\delta,
	\end{equation}
	if the $G_{t_i}$ used to estimate $\hat{Q}^{\pi_b}_\mathcal{D}$ are
	independent.
\end{lemma}

The advantage of this new error function is the improved asymptotic bound, especially for low variance returns.
There are also two disadvantages. The first is that if only a few observations are available, the bound is worse than the bound achieved by the Hoeffding's bound. The second is that the value $G_{max}$ is used in the definition of $e_Q^B$. However, the maximal return is not necessarily known. There are various possible approaches to solve this issue. Either by estimating it from the observations $G_t$, lower bounding it by $\frac{1}{2}(\max_tG_t-\min_tG_t)$ or approximating it using $\frac{R_{max}}{1-\gamma}$. The first approach might be very accurate in practice when enough data is available, the second approach might be less powerful, however, Lemma \ref{lem:mpeb} is certain to hold for any lower bound on $G_{max}$ and $\frac{R_{max}}{1-\gamma}$ might be a good a priori known approximation for some MDPs and it would also work to substitute $G_{max}$ in this whole theory by $\frac{R_{max}}{1-\gamma}$, which makes also Theorem \ref{th:theorem-1} more interpretable but less powerful.

In Section \ref{sec:limitations-of-the-theory-in-practice} we empirically investigate how the error function relying on the Maurer and Pontil bound compare to the one relying on Hoeffding's bound.

\section{Experiments}\label{sec:experiments}
We test the adapted Soft-SPIBB algorithms against Basic RL (classical Dynamic Programming \cite{sutton_reinforcement_nodate} on the MLE MDP $\hat{M}$), Approx-Soft-SPIBB \cite{nadjahi_safe_2019}, its predecessors  $\Pi_b$- and $\Pi_{\leq b}$-SPIBB \cite{laroche_safe_2019}, DUIPI \cite{jabin_uncertainty_2010}, RaMDP \cite{petrik} and R-MIN, the pessimistic adaptation of R-MAX \cite{Brafman2002RMAXA}.
    
Since the adaptation of DUIPI to avoid unknown state-action pairs, as described in Section \ref{sec:duipi}, proved in these experiments to be an important and consistent improvement, we display only the results of this adaptation of DUIPI in Section \ref{sec:results}. We also checked if the error function relying on the bound by Maurer and Pontil \cite{mpeb} improved the performance of the Soft-SPIBB algorithms. However, it did not provide a significant improvement and, thus, we also only show the results for the algorithms applying the error function based on Hoeffding's inequality.

\subsection{Experimental Settings}

We use two different benchmarks for our comparison. The first one is the Random MDPs benchmark already used in Laroche et al. \cite{laroche_safe_2019} and Nadjahi et al. \cite{nadjahi_safe_2019}. As the second benchmark we use the Wet Chicken benchmark \cite{alippi_efficient_2009} which depicts a more realistic scenario.

We perform a grid-search to choose the optimal hyper-parameter for each algorithm for both benchmarks. Our choices can be found in the table below. The hyper-parameter selection is the main difference to the experiments we conduct in Section \ref{sec:limitations-of-the-theory-in-practice} as the hyper-parameters here are solely optimized for performance without considerations of the safety guarantees, while we choose the hyper-parameter in Section \ref{sec:limitations-of-the-theory-in-practice} such that the bounds are meaningful.
\begin{table}[h]
 \centering
\caption{Chosen hyper-parameters for both benchmarks~\cite{scholl_icaart22}.} \label{tab:hyper-parameter}
\begin{tabular}{|c|c|c|}
	\hline
	\textbf{Algorithms} & \makecell{\textbf{Random} \textbf{MDPs}} & \textbf{Wet Chicken} \\
	\hline
    Basic RL & - & - \\
	\hline
	RaMDP & $\kappa=0.05$ & $\kappa=2$ \\
	\hline
	R-MIN & $N_\wedge=3$ & $N_\wedge=3$ \\
	\hline
	DUIPI & $\xi=0.1$ & $\xi=0.5$ \\
	\hline
	$\Pi_b$-SPIBB & $N_\wedge=10$ & $N_\wedge=7$\\
	\hline
    $\Pi_{\leq b}$-SPIBB &$N_\wedge=10$ & $N_\wedge=7$\\
	\hline
	\makecell{Approx-Soft-SPIBB} & $\delta=1,\,\epsilon=2$ & $\delta=1,\,\epsilon=1$\\
	\hline
	\makecell{Adv-Approx-Soft-SPIBB (ours)} & $\delta=1,\,\epsilon=2$ & $\delta=1,\,\epsilon=1$\\
	\hline
	\makecell{Lower-Approx-Soft-SPIBB (ours)} & $\delta=1,\,\epsilon=1$ & $\delta=1,\,\epsilon=0.5$\\
	\hline
\end{tabular}

\end{table}

In both experiments we conduct $10,000$ iterations for each algorithm. As we are interested in Safe Policy Improvement, we follow Chow et al. \cite{Chow2015RiskSensitiveAR}, Laroche et al. \cite{laroche_safe_2019}, and Nadjahi et al. \cite{nadjahi_safe_2019} and consider besides the mean performance also the 1\%-CVaR (Critical Value at Risk) performance, which is the mean performance over the worst 1\% of the runs.

\subsubsection{Random MDP Benchmark}
First we consider the grid-world like Random MDPs benchmark introduced in Nadjahi et al. \cite{nadjahi_safe_2019} which generates a new MDP in each iteration. The generated MDPs consist of 50 states, including an initial state (denoted by 0) and a final state. In every non-terminal state there are four actions available and choosing one leads to four possible next states. All transitions yield zero reward except upon entering the terminal state, which gives a reward of 1. As the discount factor is chosen as $\gamma=0.95$, maximizing the return is equivalent to finding the shortest route to the terminal state. 

The baseline policy on each MDP is computed such that its performance is approximately $\rho_{\pi_b}=V_{M^*}^{\pi_b}(0)=\eta V_{M^*}^{\pi_*}(0)+(1-\eta)V_{M^*}^{\pi_u}(0)$, where $0\leq\eta\leq1$ is the baseline performance target ratio interpolating between the performance of the optimal policy $\pi_*$ and the uniform policy $\pi_u$. The generation of the baseline policy starts with a softmax on the optimal action-value function and continues with adding random noise to it, until the desired performance is achieved \cite{nadjahi_safe_2019}. To counter the effects from incorporating knowledge about the optimal policy, the MDP is altered after the generation of the baseline policy by transforming one regular state to a terminal one which also yields a reward of 1.

The performances are
normalized to make them more comparable between different runs by calculating 
$\bar{\rho}_\pi= \frac{\rho_\pi-\rho_{\pi_b}}{\rho_{\pi_*}-\rho_{\pi_b}}$.
Thus, $\bar{\rho}_\pi<0$ means a worse performance than the baseline policy,
$\bar{\rho}_\pi>0$ means an improvement w.r.t.\ the baseline policy and
$\bar{\rho}_\pi=1$ means the optimal performance was reached.

\subsubsection{Wet Chicken Benchmark}

    The second experiment uses the more realistic Wet Chicken benchmark \cite{alippi_efficient_2009} because of its heterogeneous stochasticity. Figure \ref{fig:wetchicken2} visualizes the setting of the Wet Chicken benchmark. The basic idea behind it is that a person floats in a small boat on a river. The river has a waterfall at one end and the goal of the person is to stay as close to the waterfall as possible without falling down. Thus, the closer the person gets to the waterfall the higher the reward gets, but upon falling down they start again at the starting place, which is as far away from the waterfall as possible. It is modeled as a non-episodic MDP.
	
	\begin{figure}[h!]
		\centering
		\includegraphics[width=4cm]{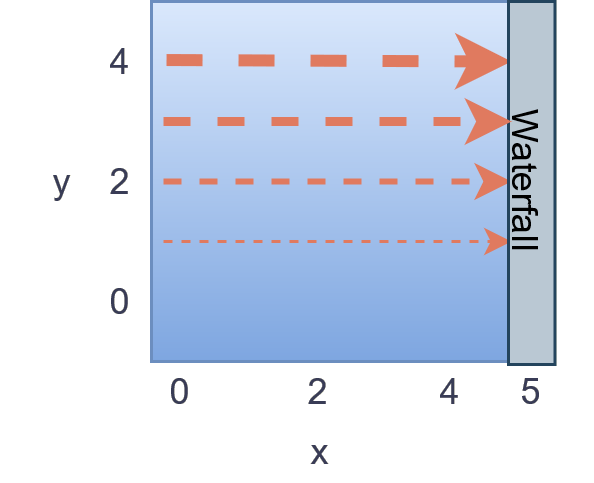}
		\vspace*{-3mm}
		\caption{The setting of the Wet Chicken benchmark used for reinforcement learning. The boat starts at $(x,y)=(0,0)$ and starts there again upon falling down the waterfall at $x=5$. The arrows show the direction and strength of the stream towards the waterfall. Additionally, turbulence is stronger for small $y$. The goal for the boat is to stay as close as possible to the waterfall without falling down~\cite{scholl_icaart22}.}
		\label{fig:wetchicken2}
	\end{figure}
	
	The whole river has a length and width of 5, so, there are 25 states. The
	starting point is $(x,y)=(0,0)$ and the waterfall is at $x=5$. The position of
	the person at time $t$ is denoted by the pair $(x_t,y_t)$. The river itself has
	a turbulence ($b_t=3.5 - v_t$) which is stronger for small $y$ 
	and a stream ($v_t=y_t\frac{3}{5}$) towards the waterfall which 
	is stronger for larger $y$. The effect of the turbulence is stochastic; so, let
	$\tau_t\sim U(-1,1)$ be the parameter describing
	the stochasticity of the turbulence at time $t$.

The person can choose from five actions listed below with $a_x$ and $a_y$ describing the
	influence of an action on $x_t$ and $y_t$, respectively:
	\begin{itemize}
		\item Drift: The person does nothing, $(a_x,a_y)=(0,0)$.
		\item Hold: The person paddles back with half their power
		$(a_x,a_y)=(-1,0)$.
		\item Paddle back: The person wholeheartedly paddles back,
		$(a_x,a_y)=(-2,0)$.
		\item Right: The person goes to the right parallel to the waterfall, $(a_x,a_y)=(0,1)$.
		\item Left: The person goes to the left parallel to the waterfall, $(a_x,a_y)=(0,-1)$.
	\end{itemize}
	The new position of the person assuming no river constraints is then calculated
	by 
	\begin{equation} \label{eq:wet-chicken-rounding}
		(\hat{x},\hat{y}) = ( {\rm round}(x_t + a_x + v_t + \tau_t  b_t),
		 {\rm round}(x_t + a_y)),
	\end{equation}
	where the $round$ function is the usual one, i.e., a number is getting rounded
	down if the first decimal is 4 or less and rounded up otherwise. Incorporating
	the boundaries of the river yields the new position as
	\begin{equation}
		x_{t+1}=\begin{dcases}
			\hat{x},& \text{if } 0\leq\hat{x}\leq4 \\
			0, & \text{otherwise}
		\end{dcases}
	\text{ and }
		y_{t+1} = \begin{dcases}
			0,& \text{if } \hat{x}>4 \\
			4,& \text{if } \hat{y}>4 \\
			0,& \text{if } \hat{y}>0 \\
			\hat{y}, & \text{otherwise}
		\end{dcases}.
	\end{equation}
	
	As the aim of this experiment is to have a realistic setting for Batch RL, we use a realistic behavior policy. Thus, we do not incorporate any knowledge about the transition probabilities or the optimal policy as it has been done for the Random MDPs benchmark. Instead we devise heuristically a policy, considering the overall structure of the MDP.
	
	Our behavior policy follows the idea that the most beneficial state might lie in the middle of the river at $(x,y)=(2,2)$. This idea stems from two trade-offs. The first trade-off is between low rewards for a small $x$ and a high risk of falling down for a big $x$ and the second trade-off is between a high turbulence and low velocity for a low $y$ and the opposite for big $y$. To be able to ensure the boat stays at the same place turbulence and velocity should both be limited. 
	
	This idea is enforced through the following procedure. If the boat is not in the state $(2,2)$, the person tries to get there and if they are already there, they use the action \textit{paddle back}. Denote this policy with $\pi_b'$. We cannot use this exact policy, as it is deterministic, i.e., in every state there is only one action which is chosen with probability 1.
	This means that for each state there is at most 1 action for which data is
	available when observing this policy. This is countered
	by making $\pi_b'$ $\epsilon$-greedy, i.e., define the behavior policy $\pi_b$
	as the mixture
	\begin{equation}
		\pi_b = (1-\epsilon)\pi_b' + \epsilon \pi_u
	\end{equation}
	where $\pi_u$ is the uniform policy which chooses every action in every state
	with the same probability. $\epsilon$ was chosen to be 0.1 in the following experiments. By calculating the transition matrix of the Wet Chicken benchmark, it is possible to determine the performances of our polices exactly and compute the optimal policy. This helps to get a better grasp on the policies' performances: The baseline policy with $\epsilon=0.1$ has a performance of 29.8, the uniform policy $\pi_u$ of 20.7 and the optimal policy of 43.1.

\subsection{Results} \label{sec:results}
\begin{figure}[h]
\centering
	\begin{subfigure}{1\textwidth}
		\centering
		\includegraphics[width=0.9\linewidth]{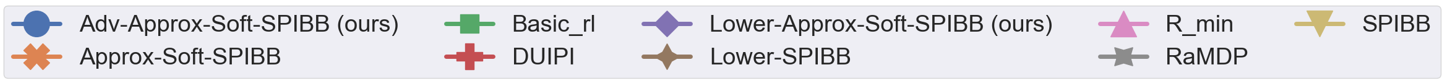}
		\vspace{0.2cm}
	\end{subfigure}
\centering
\begin{subfigure}{.5\textwidth}
  \centering
  \includegraphics[width=1\linewidth]{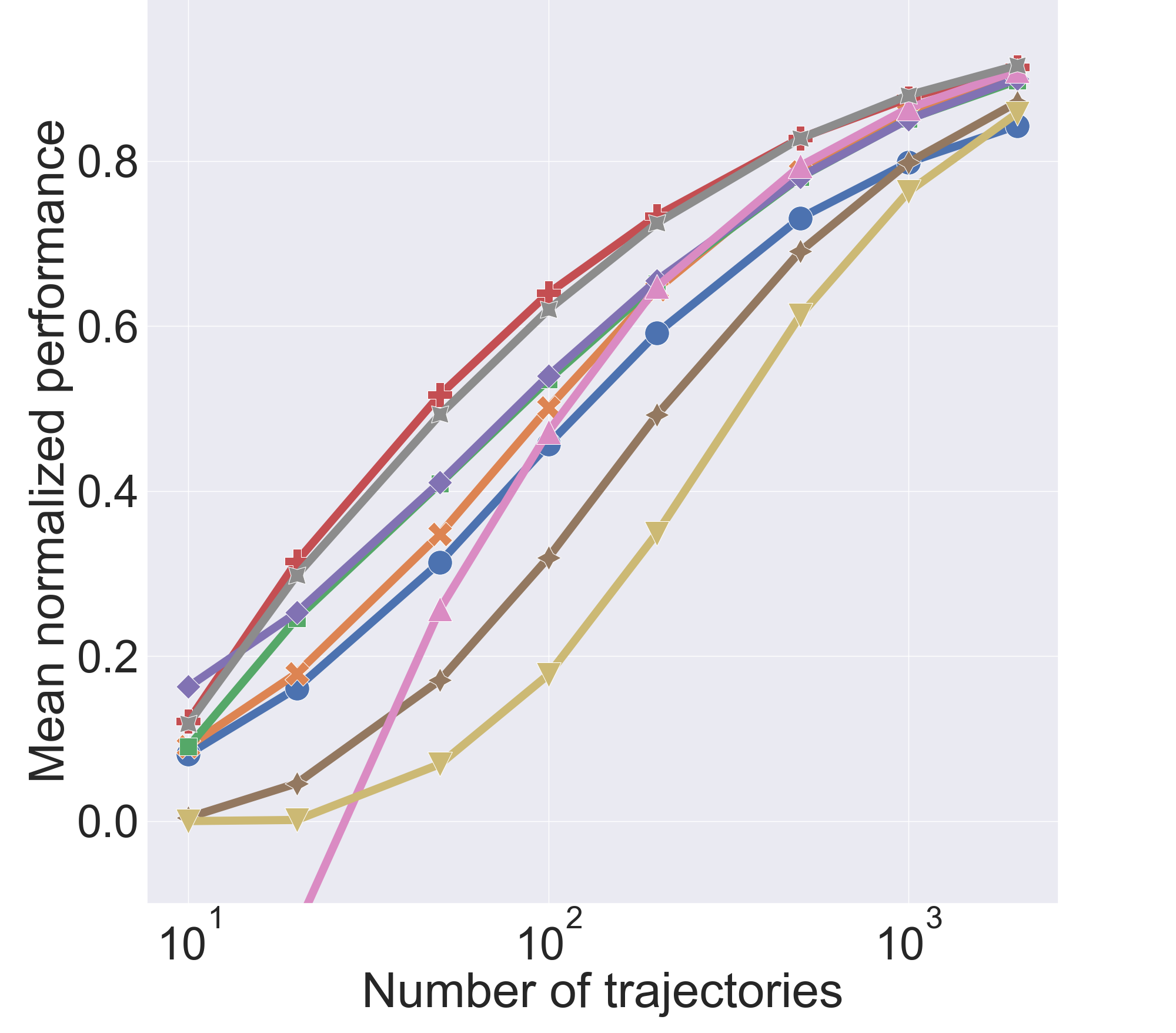}
		\caption{Mean}
		\label{fig:random-mdps-mean}
\end{subfigure}%
\begin{subfigure}{.5\textwidth}
  \centering
  \includegraphics[width=1\linewidth]{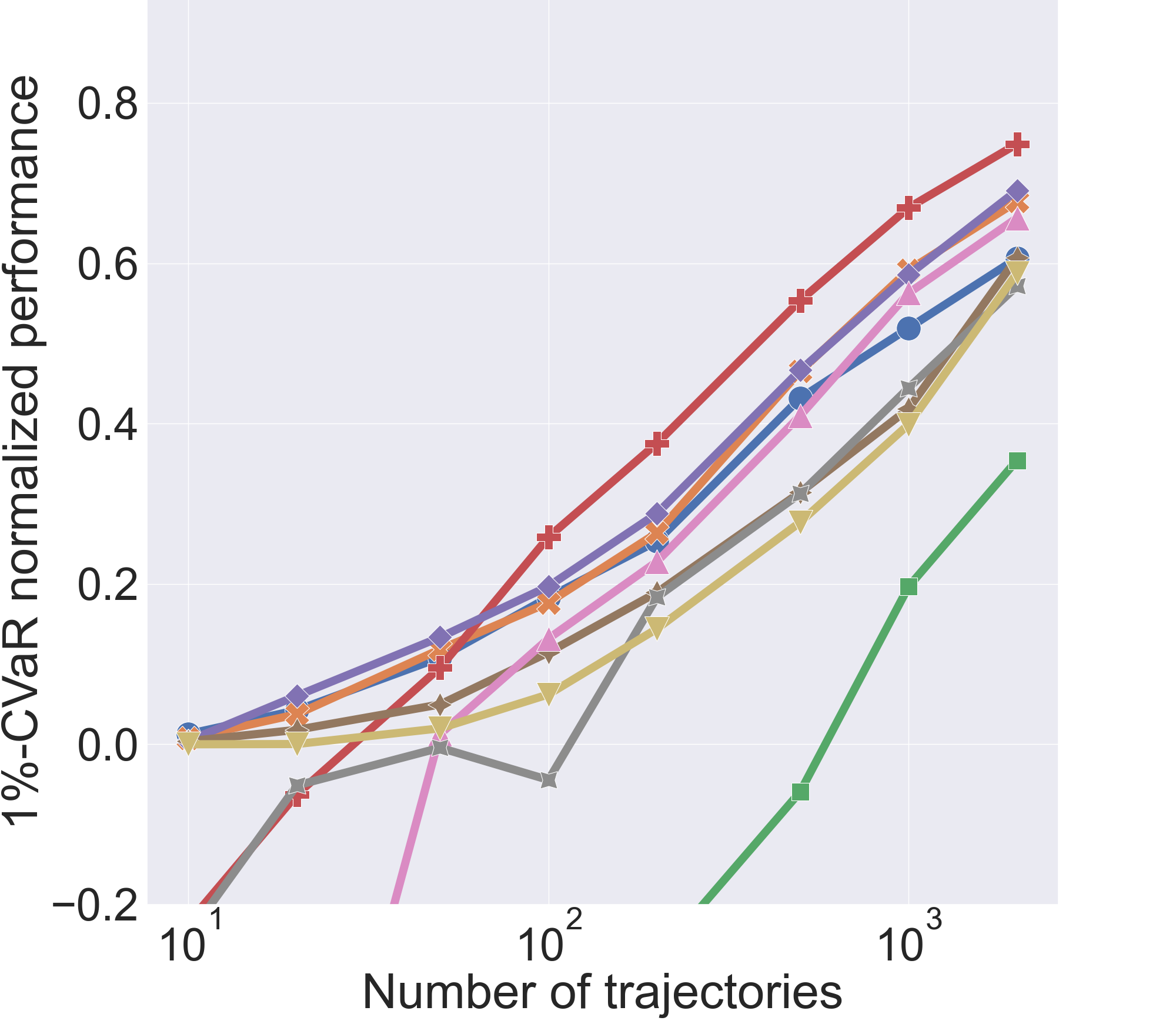}
		\caption{1\%-CVaR}
		\label{fig:random-mdps-cvar}
\end{subfigure}
\caption{Mean (a) and 1\%-CVaR (b) normalized performance over 10,000 trials on the Random MDPs benchmark for $\rho_{\pi_b}=0.9$. In the context of SPI the focus lies on the 1\%-CVaR. The mean performance is dominated by the algorithms applying a penalty on the action-value function, while the restricting algorithms are winning for few data points in the risk-sensitive 1\%-CVaR measure and only lose to DUIPI in the long run. Among the SPIBB class, Lower-Approx-Soft-SPIBB shows the best performance in both runs~\cite{scholl_icaart22}.}
	\label{fig:random-mdps}
\end{figure}
	\begin{figure}[h]
    \centering
    	\begin{subfigure}{1\textwidth}
    		\centering
    		\includegraphics[width=\textwidth]{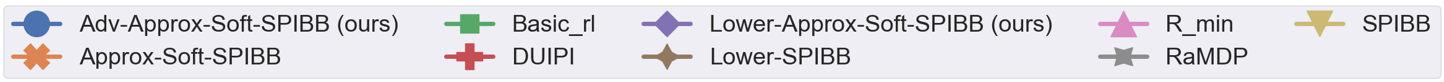}
    		\vspace{0.2cm}
    	\end{subfigure}
    \centering
    \begin{subfigure}{.5\textwidth}
      \centering
      \includegraphics[width=1\linewidth]{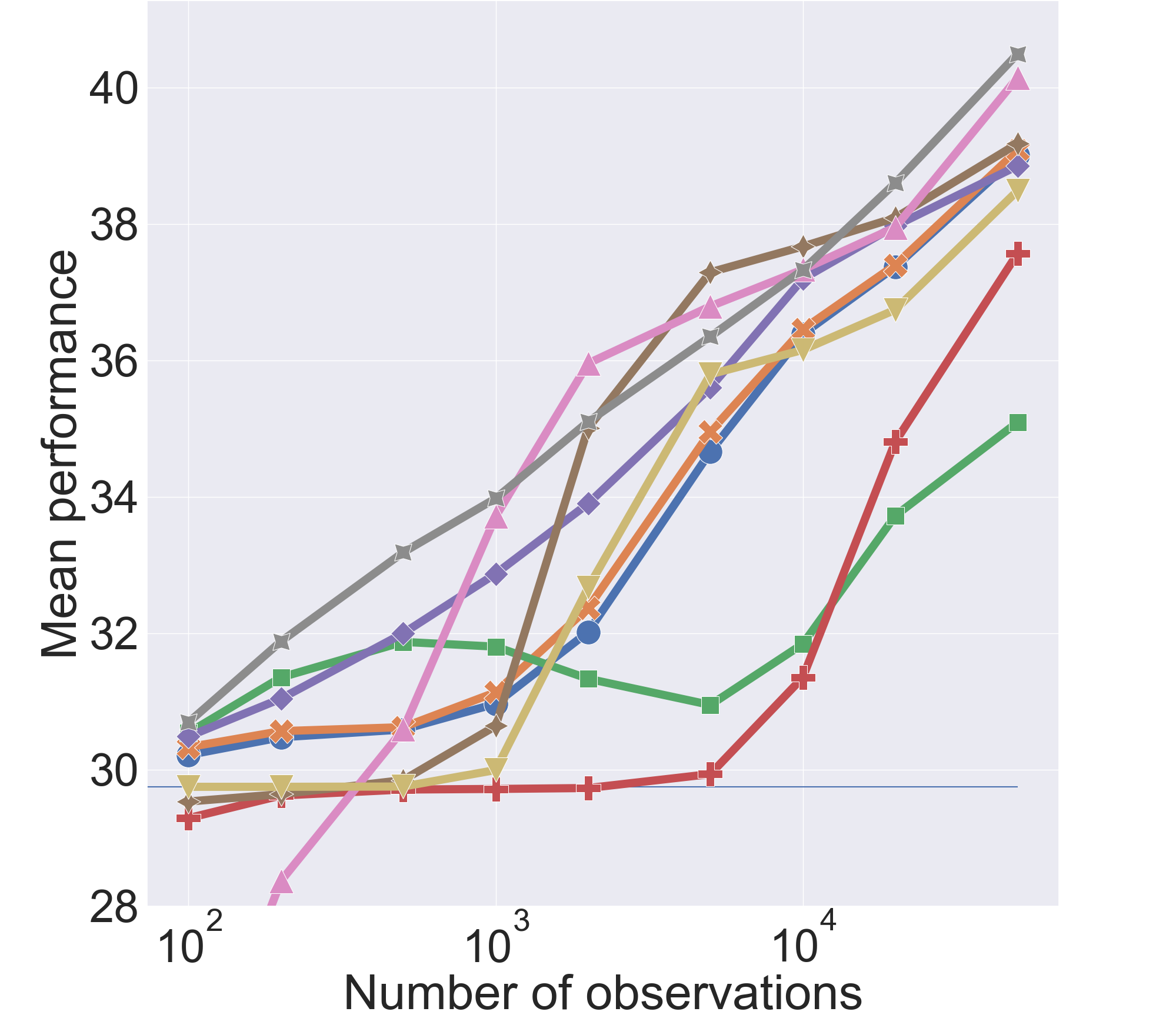}
    		\caption{Mean}
    		\label{fig:wet-chicken-mean}
    \end{subfigure}%
    \begin{subfigure}{.5\textwidth}
      \centering
      \includegraphics[width=1\linewidth]{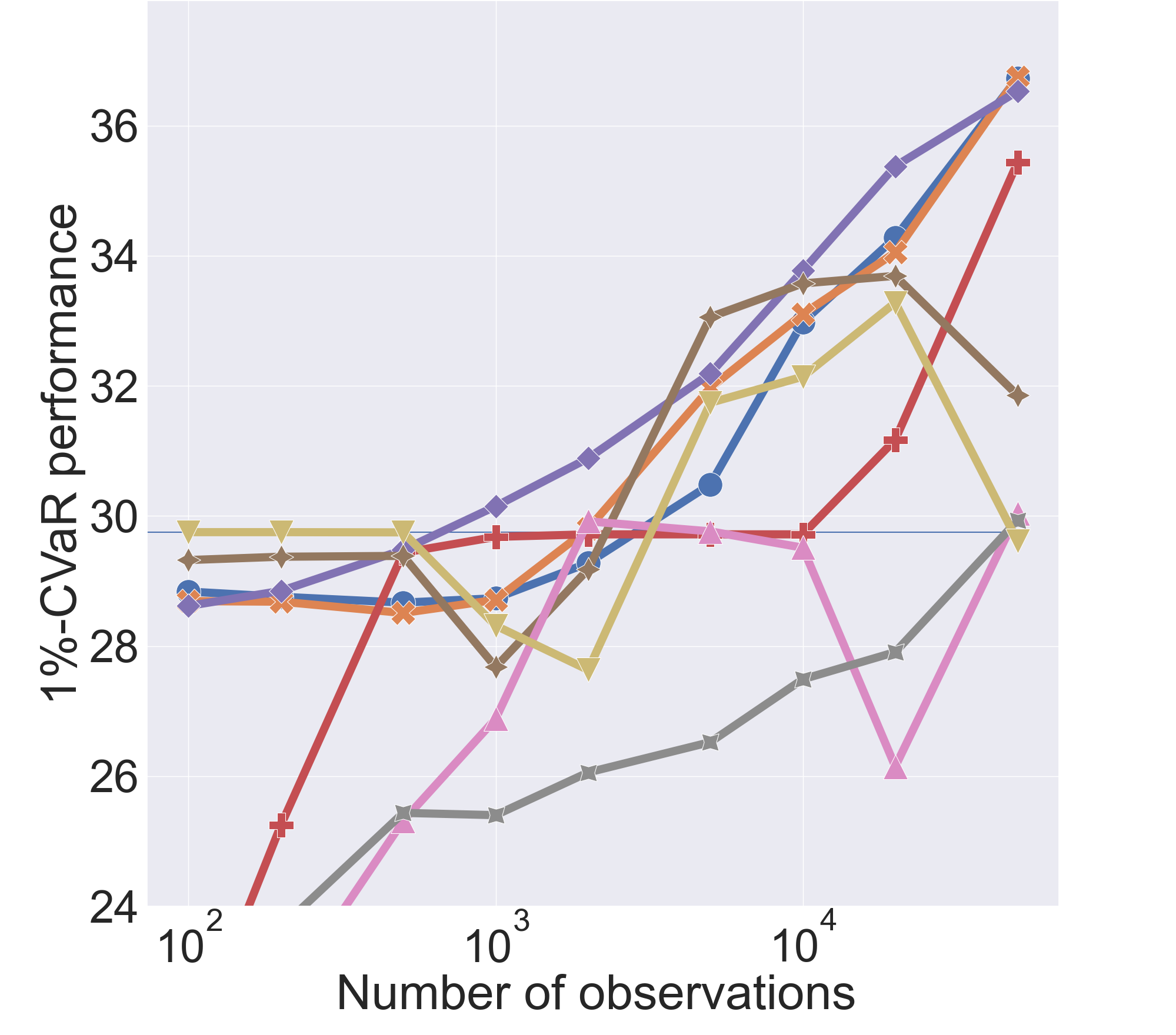}
    		\caption{1\%-CVaR}
    		\label{fig:wet-chicken-cvar}
    \end{subfigure}
    	\caption{Mean (a) and 1\%-CVaR (b) performance over 10,000 trials on the Wet Chicken benchmark for $\epsilon=0.1$ for the baseline policy. The mean performance is dominated by RaMDP, while the restricting algorithms are winning in the risk-sensitive 1\%-CVaR measure. Among the SPIBB class, Lower-Approx-Soft-SPIBB shows the best mean and 1\%-CVaR performance~\cite{scholl_icaart22}.}
    	\label{fig:wet-chicken}
    \end{figure}

The results with optimized hyper-parameters as shown in Table \ref{tab:hyper-parameter} can be seen in Figure \ref{fig:random-mdps} and \ref{fig:wet-chicken} for the Random MDPs and Wet Chicken benchmark, respectively. The performances for both benchmarks are very similar, except for DUIPI which apparently has strong problems to estimate the variance in the complex environment of the Wet Chicken benchmark. In general, the algorithms penalizing the action-value function dominate the mean performance, especially when large amounts of data are available, but they lose against the class of restricting algorithms when it comes to the 1\%-CVaR performance, which is the more interesting metric when talking about safety. The strongest performing algorithm among the class of restricting algorithm with respect to both metrics\textemdash{}thus, also the safest among both classes\textemdash{}and in both benchmarks is Lower-Approx-Soft-SPIBB. It is followed by Approx-Soft-SPIBB which is slightly better than its theoretical adaptation Adv-Approx-Soft-SPIBB. The worst algorithms from this class are the original SPIBB algorithms: $\Pi_b$-SPIBB (SPIBB) and $\Pi_{\leq b}$-SPIBB (Lower-SPIBB).

Interestingly, Basic RL is outperformed by various algorithms in the mean performance, which might be surprising as the others are
intended for safe RL instead of an optimization of their mean performance. The
reason for this might be that considering the uncertainty of the action-value
function is even beneficial for the mean performance. Also, considering its 1\%-CVaR performance\textemdash{}not even visible in Figure \ref{fig:wet-chicken-cvar} due to its inferior performance\textemdash{}the safety advantage of the SPI algorithms becomes very clear.

These two experiments demonstrate that restricting the set of policies instead of adjusting the action-value function can be very beneficial for the safety aspect of RL, especially in complex environments and for a low number of observations. On the contrary, from a pure mean performance point of view it is favorable to rather adjust the action-value function. 

\section{Critical Discussion of the Safety Bounds}\label{sec:limitations-of-the-theory-in-practice}
    In this section we want to check the strength of the safety bounds in the experiments. The only theoretically safe algorithms among those in the last section are Adv-Approx-Soft-SPIBB, $\Pi_b$-SPIBB and DUIPI. However, DUIPI only uses an approximation of the variance and additionally assumes that the action-value function is normally distributed. This has to be kept in mind when comparing
	the different bounds. To test these bounds we repeat the experiments on the Wet Chicken benchmark and this time choose the hyper-parameters such that the bound is a useful one.
	
	\paragraph{Adv-Approx-Soft-SPIBB}
    For Adv-Approx-Soft-SPIBB the safety guarantee is Theorem
	\ref{th:theorem-1} and, thus, we set $\delta>0$ close to 0 and also
	$\epsilon>0$ small enough to make Equation \ref{eq:theorem-1} meaningful. All the parameters are known a priori so one can
	choose $\epsilon$ and directly know the value of the bound. As we use $\gamma=0.95$, all the returns are between 0 and 80 ($80 =
	\frac{4}{1-\gamma}$) and, therefore, $G_{max}=40$, as there is an MDP equivalent
	to the usual Wet Chicken benchmark but with centered returns and $G_{max}$ only
	needs to be an upper bound on these centered ones. We choose
	$\epsilon$ to be at most $0.01$, as this yields a lower bound of $-8$ below the baseline policy's performance.
	
	As mentioned before, $G_{max}$ is known exactly in this setting, which enables the use of the error function relying on the Maurer and Pontil bound for Adv-Approx-Soft-SPIBB. In Section \ref{sec:advantageous-soft-spibb} we propose this new error function specifically for its improved asymptotic tightness and consequently we will test it out in this experiment as well.
	
	\paragraph{$\Pi_b$-SPIBB} In Section \ref{sec:spibb} we mentioned that $\Pi_b$-SPIBB produces policies which
	are $\xi$-approximate ($1-\delta$)-safe policy improvements w.r.t. the behavior
	policy $\pi_b$, where 
	\begin{equation} \label{eq:safety-of-spibb}
		\xi=\frac{4V_{max}}{1-\gamma}
		\sqrt{\frac{2}{N_\wedge}\log\frac{2|\mathcal{S}||\mathcal{A}|2^{|\mathcal{S}|}}{\delta}}.
	\end{equation}
	So, one has to choose the hyper-parameter $N_\wedge$, which is the number of
	visits a state-action pair needs for the behavior policy to be changed at this
	state-action pair, such that this lower bound becomes meaningful. In the Wet
	Chicken benchmark, $|\mathcal{S}|=25$, $|\mathcal{A}|=5$ and $V_{max}\approx 20$
	(referring to an MDP with centered performances). We use
	$\gamma=0.95$ and want $\delta\leq0.05$. To choose $N_\wedge$, Equation	\ref{eq:safety-of-spibb} is transformed to 
	\begin{equation} \label{eq:safety-of-spibb-transformed}
		N_\wedge =
		\frac{32V_{max}\log\left(\frac{2\mathcal{S}||\mathcal{A}|2^{|\mathcal{S}|}}{\delta}\right)}{\xi^2(1-\gamma)^2}.
	\end{equation}
	To gain again a lower bound of at least $-8$, one has to set $\xi=8$ and
	inserting all the values into Equation \ref{eq:safety-of-spibb-transformed}
	yields $N_\wedge=1,832,114$ for $\delta=0.05$. For a smaller $\delta$ this is even bigger. It is obvious that $\Pi_b$-SPIBB does not work for a $N_\wedge$
	that high, as the output policy would always be the behavior policy unless there
	are far more than $10,000,000$ observations available, which seems unreasonable for such a small MDP. For this reason, we exclude $\Pi_b$-SPIBB
	from this experiment. 
	
	\paragraph{DUIPI} Lastly, DUIPI's safety stems from the calculation of the variance of the action-value
	function by assuming a distribution on the action-value function and then applying the bound through Equation \ref{eq:duipi-safety-bound}. For example, for a lower
	bound which applies to $99\%$ one has to choose $\xi>2.33$. Contrary to Adv-Approx-Soft-SPIBB it is unknown in advance what the lower bound is, as this depends
	on the variance which depends on the policy itself. 
	
		\begin{table}[h]
		    \caption{Performance and bounds of the provably safe algorithms.} \label{tab:safety-benchmark}
     \centering
    \begin{tabular}{|c|c|c|c|c|c|c|}
    	\hline
    	\textbf{Algorithms} & 
    	\multicolumn{2}{c|}{\makecell{\textbf{Adv-Approx-} \\ \textbf{Soft-SPIBB} \\ \textbf{(Hoeffding)}}} & \multicolumn{2}{c|}{\makecell{\textbf{Adv-Approx-} \\ \textbf{Soft-SPIBB} \\ \textbf{(Maurer and Pontil)}}} & 
    	\multicolumn{2}{c|}{\textbf{DUIPI}} \\
    	\hline
        \makecell{\textbf{Length} \\ \textbf{trajectory}}  & 
        \makecell{\textbf{1\%-CVaR} \\ \textbf{performance} } & \makecell{\textbf{Bound} } & 
        \makecell{\textbf{1\%-CVaR} \\ \textbf{performance} } & \makecell{\textbf{Bound} } & 
        \makecell{\textbf{1\%-CVaR} \\ \textbf{performance} } & \makecell{\textbf{Bound} } \\
    	\hline
    	5,000 & 29.6 & 21.5 & 29.6 & 21.5 & 29.7 & 25.2 \\
    	\hline
    	10,000 & 29.7 & 21.5 & 29.7 & 21.5 & 29.7 & 26.5 \\
    	\hline
    	50,000 & 30.1 & 21.5 & 30.3 & 21.5 & 29.7 & 28.2 \\
    	\hline
    	100,000 & 30.4 & 21.5 & 30.9 & 21.5 & 30.1 & 29.5 \\
    	\hline
        500,000 & 31.4 & 21.5 & 33.7 & 21.5 & 37.7 & 37.2 \\
    	\hline
    	1,000,000 & 32.1 & 21.5 & 35.4 & 21.5 & 38.0 & 37.6 \\
    	\hline
    \end{tabular}
    \end{table}
	
	\bigskip
	
	Table \ref{tab:safety-benchmark} shows the algorithms with their respective 1\%-CVaR performances over 10,000 runs and the safety bounds. The hyper parameters were chosen for each of the algorithms such that the bound holds at least for 99\% of the runs, however, in these experiments they actually held all of the time, indicating that the bounds might be too loose. The baseline policy is as described in the previous section, however, it is 0.2-greedy this time and has a performance of 29.5.
	
	These experiments show that for all algorithms a huge amount of data is
	necessary to get a sharp bound and produce a good policy at the same time. Adv-Approx-Soft-SPIBB applying the error function relying on Hoeffding's inequality performs arguably the worst. While still achieving the same security, the performance of the computed policies is significantly higher by using the error function relying on the Maurer and Pontil bound. DUIPI simply reproduces the greedy version of the baseline policy until a trajectory length of 50,000. As soon as 500,000 steps are in the data set, however, DUIPI strongly outperforms its competitors and also yields meaningful lower bounds on its performance. However, one has to keep in mind that DUIPI uses a first order Taylor approximation for the uncertainty propagation, approximates the full covariance matrix only by the diagonal and assumes a prior distribution on the action-value function.
    
    Adv-Approx-Soft-SPIBB is very sensitive to changes in $\epsilon$. E.g., for $\epsilon=0.005$, which yields a safety bound of 25.5, Adv-Approx-Soft-SPIBB using Hoeffding's inequality achieves only a 1\%-CVaR performance of 30.9 for a trajectory length of 1,000,000 and of 33.2 if it uses the inequality of Maurer and Pontil. Contrary to that, changes in $\delta$ have a weaker effect on its performance, as setting $\delta=0.001$ results in a 1\%-CVaR performance of 31.9 and 34.9 for Adv-Approx-Soft-SPIBB relying on the inequalities by Hoeffding and by Maurer and Pontil, respectively. DUIPI only has one parameter to control the likelihood of the bound being correct and choosing it such that the bound holds with 99.9\% (corresponding to a $\delta=0.001$), leaves the 1\%-CVaR performance for a trajectory with 1,000,000 steps unchanged but deteriorates the 1\%-CVaR performance for smaller data sets, i.e., for less than 500,000 observations the new policy will simply be the greedy baseline policy and for 500,000 the 1\%-CVaR performance becomes 35.4.
\FloatBarrier

\section{Conclusion}
\label{sec:conclusion}
In this paper, we reviewed and classified multiple SPI algorithms. We showed that the algorithms in Nadjahi et al. \cite{nadjahi_safe_2019} are not provably safe and have proposed a new version that is provably safe. Adapting their ideas, we also derived a heuristic algorithm which shows, among the entire SPIBB class on two different benchmarks, both the best mean performance and the best $1\%$-CVaR performance, which is important for safety-critical applications.
Furthermore, it proves to be competitive in the mean performance against other state of the art uncertainty incorporating algorithms and especially to outperform them in the $1\%$-CVaR performance. Additionally, it has been shown that the theoretically supported Adv-Approx-Soft-SPIBB performs almost as well as its predecessor Approx-Soft-SPIBB, only falling slightly behind in the mean performance. 

The experiments also demonstrate different properties of the two classes of SPI algorithms in Figure~\ref{fig:taxonomy}: algorithms penalizing the action-value functions tend to perform better in the mean, but lack in the 1\%-CVaR, especially if the available data is scarce.

In our experiments in Section \ref{sec:limitations-of-the-theory-in-practice}, we demonstrate that the safety guarantees only become helpful when huge amounts of data are available, which restricts their use cases for now. However, we also show that there are still possible improvements, for example, relying on tighter error functions. To explore this idea further or to look for improvements of the bound in Theorem \ref{th:theorem-1}, as already started in Scholl \cite{scholl}, would be important next steps to make SPI more applicable in practice.

Perhaps the most relevant direction of future work is to apply this framework to continuous 
MDPs, which has so far been explored by Nadjahi et al. \cite{nadjahi_safe_2019} and Scholl \cite{scholl} without theoretical safety guarantees. Apart from the theory, we hope that our observations of the two classes of SPI algorithms can contribute to the 
choice of algorithms for the continuous case.

\subsubsection*{Acknowledgements}
FD was partly funded by Deutsche Forschungsgemeinschaft (DFG, German Research Foundation),  project 468830823.  PS, CO and SU were partly funded by German Federal Ministry of Education and Research, project 01IS18049A (ALICE III).

%
%
%
\bibliographystyle{splncs04}
%
\bibliography{main}





\end{document}